\def\year{2021}\relax
\documentclass[letterpaper]{article} 
\usepackage{aaai21}  
\usepackage{times}  
\usepackage{helvet} 
\usepackage{courier}  
\usepackage[hyphens]{url}  
\usepackage{graphicx} 
\usepackage{natbib}  
\usepackage{caption} 
\usepackage{graphicx}
\usepackage{booktabs} 
\usepackage{times}
\usepackage{soul}
\usepackage{import}
\usepackage{amsmath,amsfonts,amsthm,bm}
\usepackage{makeidx}
\usepackage{algorithm}
\usepackage{url}
\usepackage{thmtools}
\usepackage{algorithmic}
\usepackage{amsmath}
\usepackage{amsthm}
\usepackage{booktabs}
\usepackage{xcolor}
\usepackage{mathdots}
\usepackage{yhmath}
\usepackage{cancel}
\usepackage{siunitx}
\usepackage{array}
\usepackage{multirow}
\usepackage{amssymb}
\usepackage{gensymb}
\usepackage{tabularx}
\usepackage{booktabs}
\usepackage{booktabs}
\usepackage{amsthm}
\usepackage[colorinlistoftodos]{todonotes}
\definecolor{blued}{RGB}{70,197,221}

\usepackage{xspace}
\usepackage{thmtools}
\usepackage[inline]{enumitem}
\usepackage{natbib}
\usepackage{mathtools}
\usepackage{dsfont}
\usepackage{algorithmic}

\newtheorem{proposition}{Proposition}
\newtheorem{definition}{Definition}

\newcommand{\nrm}[1]{\left\|#1\right\|}

\newcommand{\algname}{NOHD}
\DeclareRobustCommand{\quotes}[1]{``#1''}

\DeclareMathOperator{\EX}{\mathbb{E}}
\DeclareMathOperator{\Ham}{\mathcal{H}}
\DeclareMathOperator{\Hes}{\mathrm{J}}
\DeclareMathOperator{\Real}{\mathbb{R}}

\newcommand{\norm}[1]{\left\|#1\right\|}

\newcommand{\vtheta}{\bm{\theta}}

\newtheorem{exmp}{Example}[section]

\declaretheorem[name=Remark]{remark}

\title{Newton Optimization on Helmholtz Decomposition for Continuous Games}
\author{
    Giorgia Ramponi, Marcello Restelli \\
}
\affiliations{
    Dipartimento di Elettronica, Informazione e Bioingegneria, Politecnico di Milano \\
Piazza Leonardo da Vinci, 32, 20133, Milano, Italy \\
\{giorgia.ramponi, marcello.restelli\}@polimi.it
}
\begin{document}

\maketitle


\begin{abstract}
Many learning problems involve multiple agents that optimize different interactive functions. In these problems, standard policy gradient algorithms fail due to the non-stationarity of the setting and the different interests of each agent. In fact, the learning algorithms must consider the complex dynamics of these systems to guarantee rapid convergence towards a (local) Nash equilibrium. In this paper, we propose \algname~ (Newton Optimization on Helmholtz Decomposition), a Newton-like algorithm for multi-agent learning problems based on the decomposition of the system dynamics into its irrotational (Potential) and solenoidal (Hamiltonian) components. This method ensures quadratic convergence in purely irrotational systems and pure solenoidal systems. Furthermore, we show that \algname~ is attracted to symmetric stable fixed points in general multi-agent systems and repelled by strict saddle ones. Finally, we empirically compare the \algname's~ performance with state-of-the-art algorithms on some bimatrix games and in a continuous Gridworld environment.
\end{abstract}

\section{Introduction}
In recent years, Reinforcement Learning (RL)~\cite{sutton} methods with multiple agents~\cite{bucsoniu2010multi,zhang2019multi} have made substantial progress in solving decision-making problems such as playing Go~\cite{go}, robotic control problems \cite{lillicrap2015continuous}, playing card games~\cite{brown2019superhuman} and autonoumous driving~\cite{shalev2016safe}.
Furthermore, in other machine learning fields, powerful algorithms that optimize multiple losses have recently been proposed. Generative Adversarial Networks (GANs)~\cite{goodfellow2014generative} is an example, which achieves successful results in Computer Vision~\citep{isola2017image, ledig2017photo} and Natural Language Generation~\citep{nie2018relgan, yu2017seqgan}. On the other hand, thanks to their ability to learn in the stochastic policy space and their effectiveness in solving high-dimensional, continuous state and action problems, policy-gradient algorithms~\citep{peters2006policy} are natural candidates for use in multi-agent learning problems.
Nonetheless, multiple policy-gradient agents' interaction has proven unsuccessful in learning a set of policies that converges to a (local) Nash Equilibrium~\citep{mertikopoulos2018cycles, papadimitriou2016nash}. More than one agent leads to the failure of standard optimization methods in most games due to the non-stationarity of the environment and the lack of cooperation between the agents. \\
How to optimize multiple policy-gradient agents is a problem of theoretical interest and practical importance. Over the past two years, a growing number of papers has addressed this problem by focusing on \textit{continuous (differentiable) games}, i.e., games where the agent's objective functions are twice differentiable with respect to the policy parameters \citep{mazumdar2020gradient}. Some of them only consider the competitive setting due to the success of GANs~\citep{mescheder2017numerics}. The general case was considered only recently, where the gradient descent update rule was combined with second-order terms~\citep{balduzzi2018mechanics, letcher2018stable, foerster2018learning, schafer2019competitive}. Some of them guarantee linear convergence under specific assumptions.\\
In this paper, we study how to build a Newton-based algorithm (\algname) for learning policies in multi-agent environments. First of all, we start by analyzing two specific game classes: Potential Games and Hamiltonian Games (Section~\ref{s:preliminaries}). In Section~\ref{s:newton_for_games}, we propose a Newton-based update for these two classes of games, for which linear-rate algorithms that guarantee convergence are known, proving quadratic convergence rates. Then, we extend the algorithm to the general case, neither Hamiltonian nor Potential. We show that the proposed algorithm respects some desiderata, similar to those proposed in~\cite{balduzzi2018mechanics}: the algorithm has to guarantee convergence to (local) Nash Equilibria in (D1) Potential and (D2) Hamiltonian games; (D3) the algorithm has to be attracted by symmetric stable fixed points and (D4) repelled by symmetric unstable fixed points. Finally, in Section~\ref{s:experiments}, we analyze the empirical performance of \algname~when agents optimize a Boltzmann policy in three bimatrix games: Prisoner's Dilemma, Matching Pennies, and Rock-Paper-Scissors.
In the last experiment, we study the learning performance of \algname~ in two continuous gridworld environments. In all experiments, \algname~achieves great results confirming the quadratic nature of the update. 
The proofs of all the results presented in the paper are reported
in Appendix
 \footnote{A complete version of the paper, which includes the appendix,
is available at https://arxiv.org/pdf/2007.07804.pdf.}.

\section{Related Works}\label{s:related_works}
The study of convergence in classic convex multiplayer games has been extensively studied and analyzed~\citep{rosen1965existence, facchinei2007generalized, singh2000nash}. Unfortunately, the same algorithms cannot be used with neural networks due to the non-convexity of the objective functions. 
Various algorithms have been proposed that successfully guarantee convergence in specific classes of games: policy prediction in two-player two-action bimatrix games~\citep{zhang2010multi, song2019convergence}; WoLF in two-player two-action games~\cite{bowling2002multiagent}; AWESOME in repeated games~\cite{conitzer2007awesome}; Optimistic Mirror Descent in two-player bilinear zero-sum games~\citep{daskalakis2017training}; Consensus Optimization~\citep{mescheder2017numerics}, Competitive Gradient Descent~\citep{schafer2019competitive} and \citep{mazumdar2019finding} in two-player zero-sum games. Other recent works propose methods to learn in smooth-markets \cite{balduzzi2020smooth}, sequential imperfect information games \cite{perolat2020poincar}, zero-sum linear-quadratic games \cite{zhang2019policy}. Furthermore, the dynamics of Stackelberg games was studied by~\citet{fiez2019convergence}. Other insights on the convergence of gradient-based learning was proposed in \cite{chasnov2020convergence}. Many algorithms have also been proposed specifically for GANs~\citep{heusel2017gans, nagarajan2017gradient, gemp2018global}.\\ 
Some recent works have developed learning algorithms also for general games. The first example is the Iterated Gradient Ascent Policy Prediction (IGA-PP) algorithm, renamed LookAhead~\citep{zhang2010multi}. In~\cite{letcher2018stable} it is proved that IGA-PP converges to local Nash Equilibria not only in two-player two-action bimatrix games but also in general games. Learning with opponent learning awareness (LOLA)~\citep{foerster2018learning} is an attempt to use the other agents' functions to account for the impact of one agent's policy on the anticipated parameter update of the other agents. The empirical results show the effectiveness of LOLA, but no convergence guarantees are provided. Indeed, \citet{letcher2018stable} have shown that LOLA may converge to non-fixed points and proposed Stable Opponent Shaping, an algorithm that maintains the theoretical convergence guarantees of IGA-PP, also exploiting the opponent dynamics like LOLA.  In~\cite{balduzzi2018mechanics, letcher2019differentiable} the authors studied game dynamics by decomposing a game into its Potential and Hamiltonian components using the generalized Helmholtz decomposition. The authors propose Symplectic Gradient Adjustment (SGA), an algorithm for general games, which converges locally to stable fixed points, using the Hamiltonian part to adjust the gradient update. Instead, our algorithm uses information about the Potential component and the Hamiltonian component of the game to approximate the general game with one of these parts at each step.

\section{Preliminaries}\label{s:preliminaries}
We cast the multi-agent learning problem as a \textit{Continuous Stochastic Game}. We adapt the concept of Stochastic Game and Continuous (Differentiable) Game from \cite{foerster2018learning, balduzzi2018mechanics, ratliff2016characterization}. 
In this section, after introducing Continuous Stochastic Games, we recall the game decomposition proposed by  \citet{balduzzi2018mechanics}. Then, we describe the desired convergence points. In the end, we introduce Newton's method. 

\textbf{Continuous Stochastic Games} 
A continuous stochastic game~\citep{foerster2018learning, balduzzi2018mechanics} is a tuple $\mathcal{G} =$ $(X, U_1, \dots$ $,  U_n, f, C_1, \dots, C_n,\gamma_1,\dots,\gamma_n)$ where $n$ is the number of agents; $X$ is the set of states; $U_i$, $1 \le i \le n$, is the set of actions of agent $i$ and $U = U_1 \times \dots \times U_n$ is the joint action set; $f: X \times U \rightarrow \Delta(X)$ is the state transition probability function (where $\Delta(\Omega)$ denotes the set of probability measures over a generic set $\Omega$), $C_i: X \times U \rightarrow \mathbb{R}$ is the cost function of agent $i$, and $\gamma_i\in[0,1)$ is its discount factor. \footnote{Given a vector $\mathbf{v}$ and a matrix $H$, in the following we will denote by $\mathbf{v}^T$ the transpose of $\mathbf{v}$ and with $\nrm{\mathbf{v}}$ and $\nrm{H}$ the respective L2-norms.}
The agent's behavior is described by means of a parametric twice differentiable policy $\pi_{\vtheta_i}: X \rightarrow \Delta(U_i)$, where $\vtheta_i \in \Theta \subseteq \mathbb{R}^d$ and $\pi_{\vtheta_i}(\cdot | x)$ specifies for each state $x$ a distribution over the action space $U_i$.\footnote{To ease the notation, we will drop $\vtheta$ (e.g., $\pi_i$ instead of $\pi_{\vtheta_i}$) when not necessary.} We denote by $\vtheta$ the vector of length $nd$ obtained by stacking together the parameters of all the agents: $\vtheta=(\vtheta_1^T,\dots,\vtheta_n^T)^T$.
We denote an infinite-horizon trajectory in game $\mathcal{G}$ by $\tau=\{x_t,\mathbf{u}_t\}_{t=0}^{\infty} \in \mathbb{T}$. where $\mathbf{u}_t = (u_1(t),\dots,u_n(t))$ is the joint action at time $t$, $x_t \sim f(x_{t-1},\mathbf{u}_{t-1})$ (for $t>0$), and $\mathbb{T}$ is the trajectory space.
In stochastic games, all agents try to minimize their expected discounted cost separately, which is defined for the $i$-th agent as:
\begin{equation}
V_i(\vtheta) = \EX_{\tau} \left[C_i(\tau)\right],
\end{equation}
where $C_i(\tau)=\sum_{t=0}^{\infty}\gamma_i^t C_i(x_t,\mathbf{u}_t)$ and the expectation is with respect to the agents' policies and the transition model.
Note that the expectation depends on all the agents' policies. We do not assume the convexity of the functions $V_i(\vtheta)$.\\
We define the \textit{symultaneous gradient}~\citep{letcher2018stable} as the concatenation of the gradient of each discounted return function respect to the parameters of each player: 
\begin{equation}
\xi(\vtheta) = (\nabla_{\vtheta_1} V_1^T, \dots, \nabla_{\vtheta_n} V_n^T)^T.
\end{equation}
The \emph{Jacobian} of the game~\citep{letcher2018stable, ratliff2013characterization} $\Hes$ is an $nd \times nd$  matrix, where $n$ is the number of agents and $d$ the number of policy parameters for each agent. $\Hes$ is composed by the matrix of the derivatives of the simultaneous gradient, i.e., for each player $i$ the $i$-th row of its
hessian:
\begin{equation*}
\resizebox{\columnwidth}{!}{$
\Hes = \nabla_{\vtheta} \xi =  \begin{pmatrix}
\nabla_{\vtheta_1}^2 V_1 & \nabla_{\vtheta_1,\vtheta_2} V_1  & \cdots & \nabla_{\vtheta_1,\vtheta_n} V_1 \\
\nabla_{\vtheta_2,\vtheta_1} V_2 & \nabla_{\vtheta_2}^2 V_2 & \cdots & \nabla_{\vtheta_2,\vtheta_n} V_2 \\
\vdots  & \vdots  & \ddots & \vdots  \\
\nabla_{\vtheta_n,\vtheta_1} V_n & \nabla_{\vtheta_n,\vtheta_2} V_n & \cdots & \nabla_{\vtheta_n}^2 V_n
\end{pmatrix}.$}
\end{equation*}

\textbf{Game dynamics}
\label{dynamics}
$\Hes$ is a square matrix, not necessarily symmetric. The antisymmetric part of $\Hes$ is caused by each agent's different cost functions and can cause cyclical behavior in the game (even in simple cases as bimatrix zeros-sum games, see Figure~\ref{fig:dynamics_ham}). On the other hand, the symmetric part represents the \quotes{cooperative} part of the game.  
In~\citep{balduzzi2018mechanics}, the authors proposed how to decompose $\Hes$ in its symmetric and antisymmetric component using the Generalized Helmholtz decomposition~\citep{wills1958vector, balduzzi2018mechanics} \footnote{The Helmholtz decomposition applies to any vector field~\citep{wills1958vector}.}.
\begin{proposition}
The Jacobian of a game decomposes uniquely into two components $\Hes = S + A$, where $S = \frac{1}{2} (\Hes + \Hes^T)$ and $A = \frac{1}{2} (\Hes - \Hes^T)$. 
\end{proposition}
Components $S$ and $A$ represent the irrotational (Potential), $S$, and the solenoidal (Hamiltonian), $A$, part of the game, respectively. The irrotational component is its curl-free component, and the solenoidal one is the divergence-free one.


\textbf{Potential games} are a class of games introduced by \cite{monderer1996potential}. A game is a potential game if there exists a potential function $\phi: \mathbb{R}^{n \times d} \rightarrow \mathbb{R}$, such that: $\phi(\vtheta^{'}_i, \vtheta_{-i}) - \phi(\vtheta^{''}_i, \vtheta_{-i}) = \alpha (V_i(\vtheta^{'}_i, \vtheta_{-i}) - V_i(\vtheta^{''}_i, \vtheta_{-i}))$. A potential game is an \textit{exact} potential game if $\alpha = 1$; exact potential games have $A = \mathbf{0}$. In these games, $\Hes$ is symmetric and it coincides with the hessian of the potential function. This class of games is widely studied because in these games gradient descent converges to a Nash Equilibrium~\cite{rosenthal1973class, lee2016gradient}.  In the rest of the document, when we refer to potential games we refer to exact potential games.

\textbf{Hamiltonian games}, i.e., games with $S = \mathbf{0}$, were introduced in~\cite{balduzzi2018mechanics}. A Hamiltonian game is described by a Hamiltonian function, which specifies the conserved quantity of the game. Formally, a Hamiltonian system is fully described by a scalar function,  $\Ham : \mathbb{R}^{n \times d} \rightarrow \mathbb{R}$. The state of a Hamiltonian system is represented by the generalized coordinates $\mathbf{q}$ momentum and position $\mathbf{p}$, which are vectors of the same size. The evolution of the system is given by Hamilton's equations: 
$
\frac{d\textbf{p}}{dt} = -\frac{\partial \Ham}{\partial \textbf{q}},
\frac{d\textbf{q}}{dt} = +\frac{\partial \Ham}{\partial \textbf{p}}.
$
The gradient of $\Ham$ corresponds to $(S+A^T)\xi$ \citep{balduzzi2018mechanics}. 
In bimatrix games, Hamiltonian games coincide with zero-sum games, but this is not true in general games~\cite{balduzzi2018mechanics}.

\textbf{Desired convergence points}
In classic game theory, the standard solution concept is the Nash Equilibrium~\cite{nash1950equilibrium}. Since we focus on gradient-based methods and make no assumptions about the convexity of the return functions, we consider the concept of local Nash Equilibrium~\citep{ratliff2016characterization}. 
\begin{definition}
A point $\vtheta^*$ is a local Nash equilibrium if, $\forall  i$, there
is a neighborhood $B_i$ of $\vtheta^*_i$ such that $V_i(\vtheta_i,\vtheta^*_{-i}) \ge 
V_i(\vtheta^*_i,\vtheta^*_{-i})$ for any $\vtheta_i \in B_i$. 
\end{definition}
Gradient-based methods can reliably find local (not global) optima even in single-agent non-convex problems~\cite{lee2016gradient, lee2017first}, but they may fail to find local Nash equilibria in non-convex games.\\
Another desirable condition is that the algorithm converges into \textit{symmetric stable fixed points} \citep{balduzzi2018mechanics}.

\begin{definition}
A fixed point $\vtheta^*$ with $\xi(\vtheta^*) = 0$ is symmetric stable if $S(\vtheta^*) \succeq 0$ and $S(\vtheta^*)$ is invertible, symmetric unstable if $S(\vtheta^*) \prec 0$ and a strict saddle if $S(\vtheta^*)$ has an eigenvalue with negative real part. 
\end{definition}
Symmetric stable fixed points and local Nash equilibria are interesting solution concepts, the former from an optimization point of view and the latter from a game-theoretic perspective. 

\textbf{Newton method} Newton's method~\cite{nocedal2006numerical} guarantees, under assumptions, a quadratic convergence rate to the root of a function, which, in optimization, is the derivative of the function to be optimized. This method is based on a second-order approximation of the twice differentiable function $g(\mathbf{\vtheta})$ that we are optimizing. Starting from an initial guess $\vtheta_0$, Newton's method updates the parameters $\vtheta$ by setting the derivative of the second-order Taylor approximation of $g(\vtheta)$ to $0$:
\begin{equation}
\vtheta_{t+1} = \vtheta_{t}-\nabla^2 g(\vtheta_t)^{-1} \nabla g(\vtheta_t).
\end{equation}
For non-convex functions, the hessian $\nabla^2 g(\vtheta)$ is not necessarily positive semidefinite and all critical points are possible solutions for Newton's method. Then, Newton's update may converge to a local minimum, a saddle, or a local maximum.
A possible solution to avoid this shortcoming is to use a modified version of the inverse of the hessian, called Positive Truncated inverse (PT-inverse)~\cite{nocedal2006numerical, paternain2019newton}:
\begin{definition}[PT-inverse]\label{d:PT-inverse}
Let $H \in \Real^{n \times n}$ be a symmetric matrix, $Q \in \Real^{n \times n}$ a basis of orthogonal eigenvectors of $H$, and $\Lambda \in \Real^{n \times n}$ a diagonal matrix of corresponding eigenvalues. The $|\Lambda|_m \in \Real^{n \times n}$ is the positive definite truncated eigenvalue matrix of $\Lambda$ with parameter $m$:
\begin{equation}
(|\Lambda|_m)_{ii} = \left\{
\begin{matrix}
    |\Lambda_{ii}| \quad if |\Lambda_{ii}| \ge m, \\
    m \quad \text{otherwise.}
 \end{matrix}\right.
\end{equation} 
The PT-inverse of $H$ with parameter $m$ is the matrix $|H|^{-1}_m = Q |\Lambda|^{-1}_m Q^T$.
\end{definition}
The PT-inverse flips the sign of negative eigenvalues and truncates small eigenvalues by replacing them with $m$. Then the usage of the PT-inverse, instead of the real one, guarantees convergence to a local minimum even in non-convex functions. These properties are necessary to obtain a convergent Newton-like method for non-convex functions.

 \section{Newton for Games}\label{s:newton_for_games}
 \begin{figure*}[t]
\centering
\begin{minipage}[t]{0.43\textwidth}
\centering
\begin{minipage}[t]{0.3\textwidth}
\centering
\includegraphics[width=\textwidth]{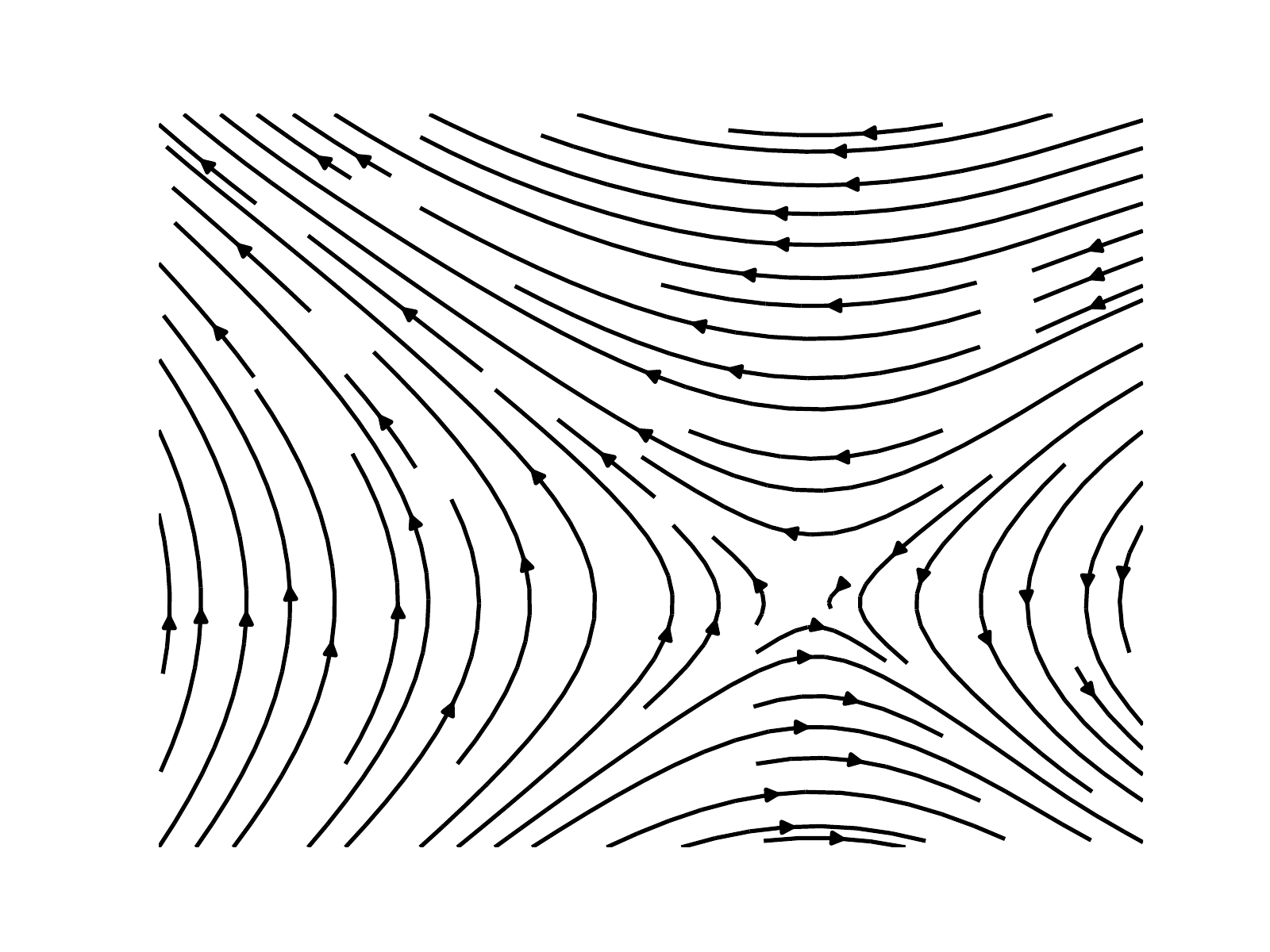}
\end{minipage}
\begin{minipage}[t]{0.3\textwidth}
\centering
\includegraphics[width=\textwidth]{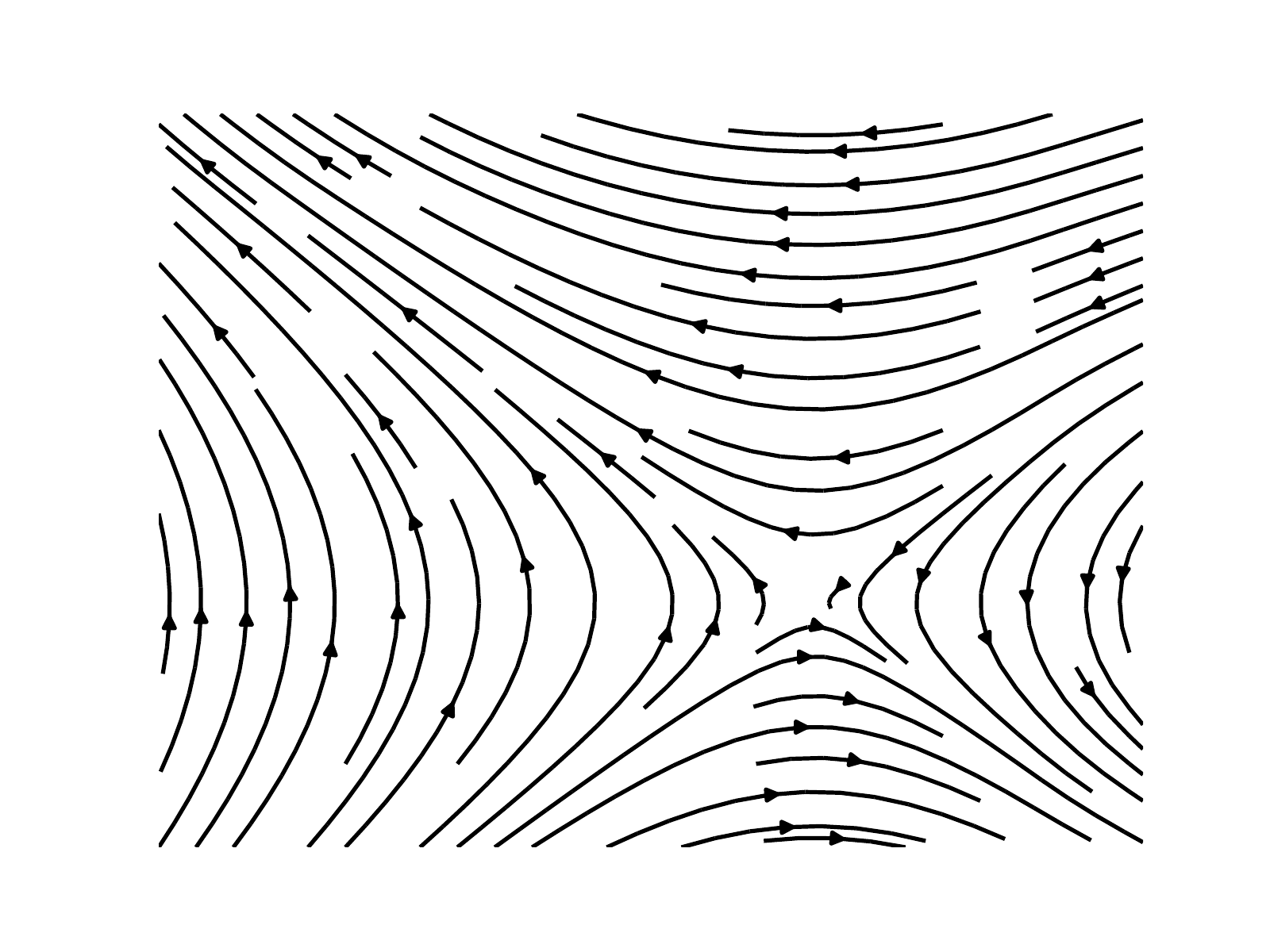}
\end{minipage}
\begin{minipage}[t]{0.3\textwidth}
\centering
\includegraphics[width=\textwidth]{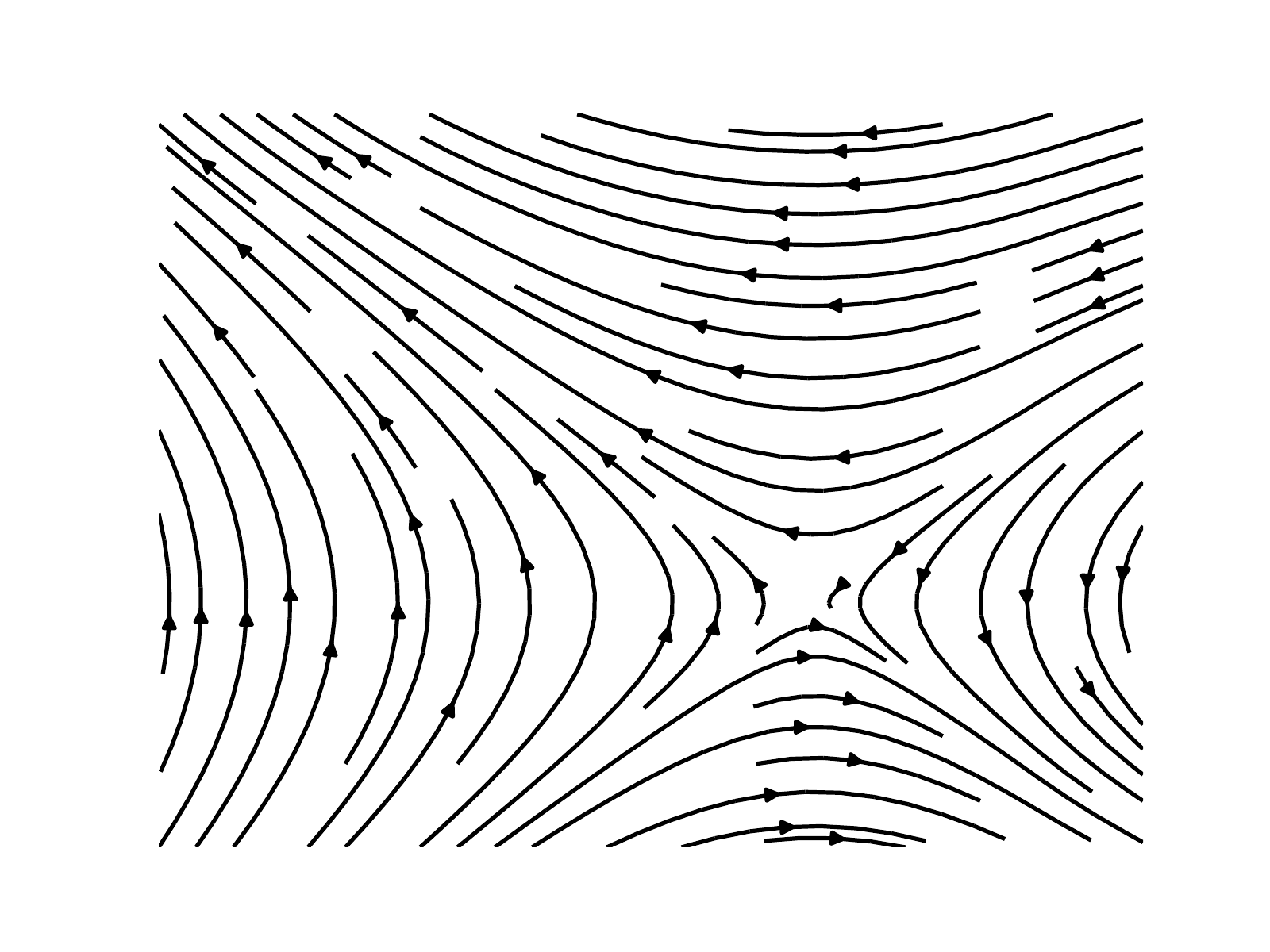}
\end{minipage}
\captionof{figure}{Phase of Potential games: from left phase dynamic of gradient descent, gradient on the potential function, and with \algname.}
\label{fig:dynamics_pot}
    \end{minipage}
    \hspace{1.5cm}
\begin{minipage}[t]{0.43\textwidth}
\centering
\begin{minipage}[t]{0.3\textwidth}
\centering
\includegraphics[width=\textwidth]{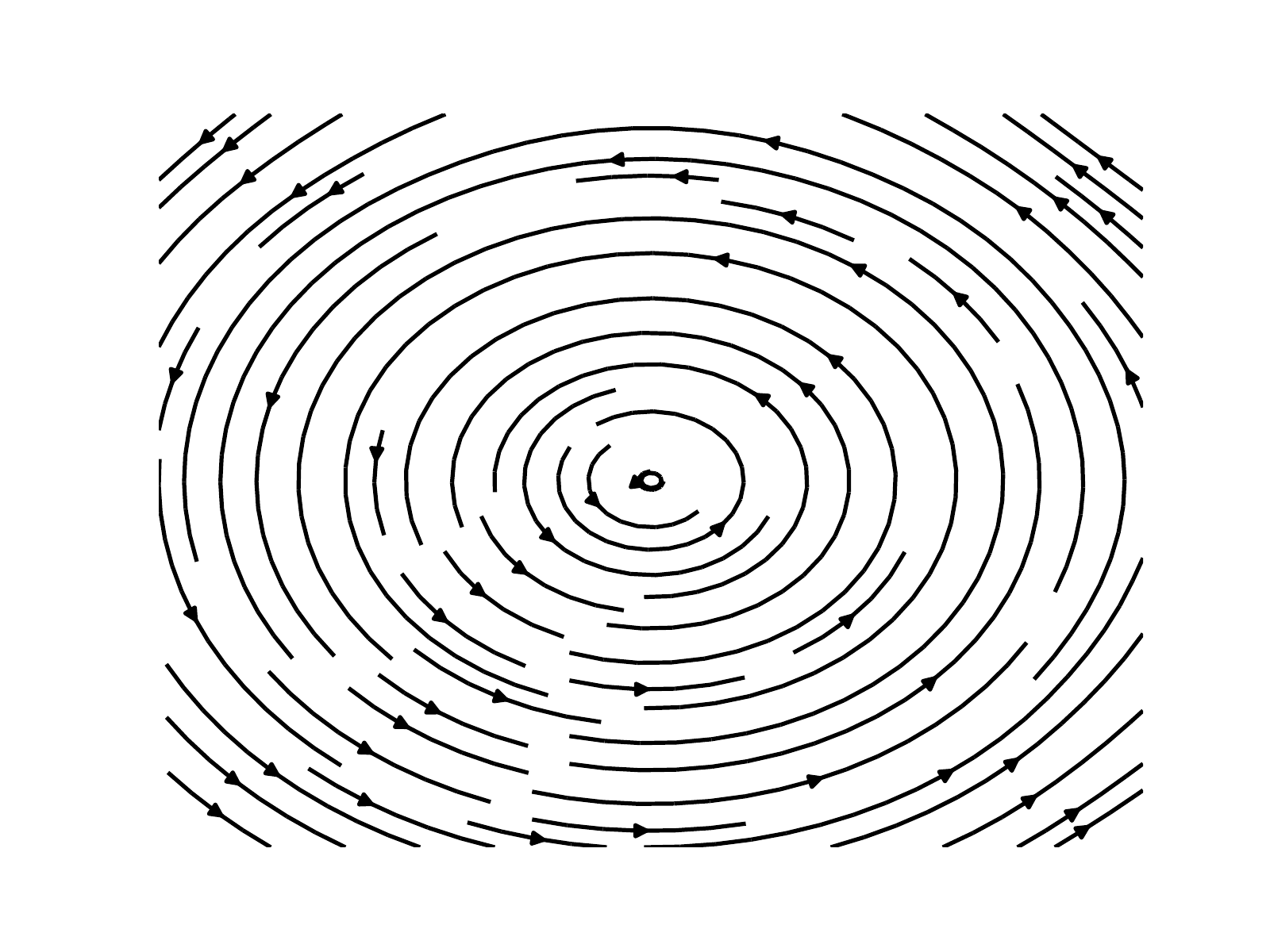}
\end{minipage}
\begin{minipage}[t]{0.3\textwidth}
\centering
\includegraphics[width=\textwidth]{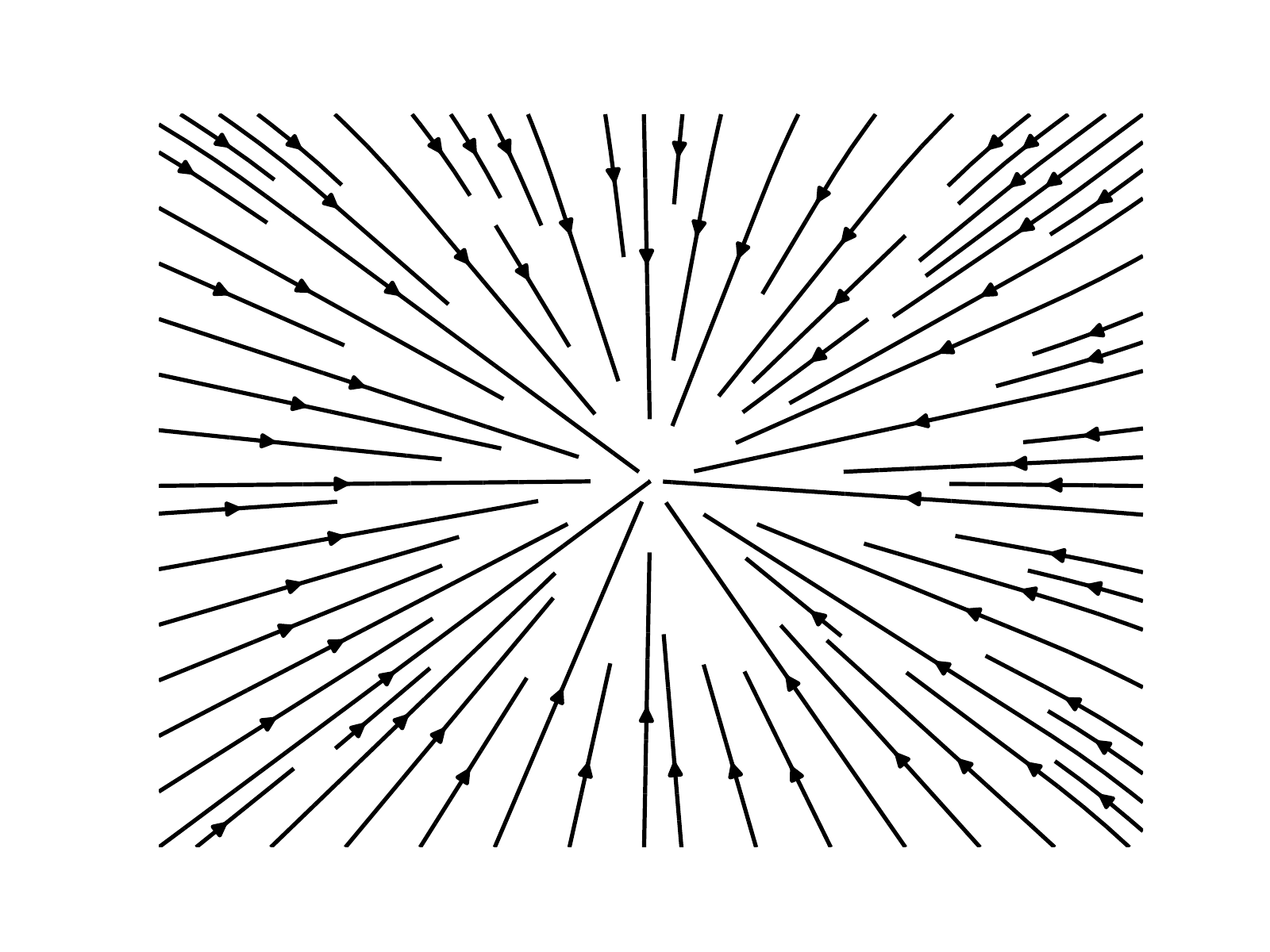}
\end{minipage}
\begin{minipage}[t]{0.3\textwidth}
\centering
\includegraphics[width=\textwidth]{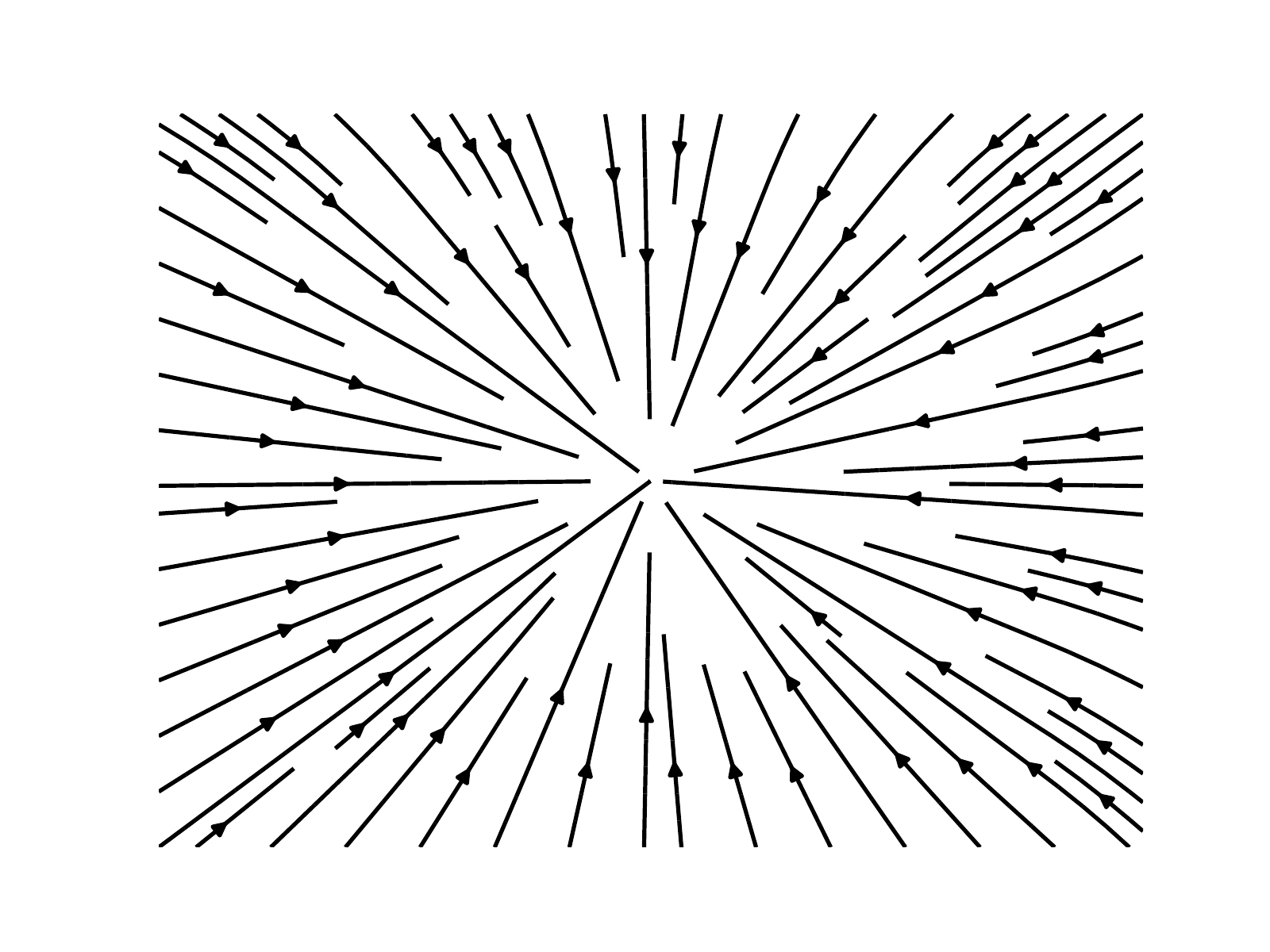}
\end{minipage}
\captionof{figure}{Phase of Hamiltonian games: from left phase dynamic of gradient descent, gradient on the Hamiltonian function, and with \algname.}
\label{fig:dynamics_ham}
\end{minipage}
\end{figure*}
In this section, we describe how to apply Newton-based methods to Continuous Stochastic Games. We start by showing how Newton's method can be applied to two game classes: Potential games and Hamiltonian Games. Then we describe an algorithm to extend Newton's method in general games.

\subsection{Newton's Method for Potential Games}
Potential games are a class of games characterized by the existence of a potential function $\phi: \mathbb{R}^{n \times d} \rightarrow \mathbb{R}$ which describes the dynamics of the system. In these games, gradient descent on the simultaneous gradient or the potential function converges to a (local) Nash Equilibrium as shown in Figure~\ref{fig:dynamics_pot}. In these games, $\Hes = S$, and the anti-symmetric component $A = \mathbf{0}$.  Considering the existence of the potential function $\phi$ (see Section~\ref{s:preliminaries}), it is sufficient to apply the Newton PT-inverse method (see Definition~\ref{d:PT-inverse}) on it to guarantee a quadratic convergence rate to a local Nash Equilibrium (see Figure~\ref{fig:dynamics_pot}). Newton's update for Potential games is:
\begin{equation}
\label{eq:potential_newton}
\vtheta_{t+1} = \vtheta_t - \eta S_m(\vtheta_t)^{-1} \xi(\vtheta_t),
\end{equation}
where $\eta$ is a learning rate and $m$ the parameter of the PT-inverse (see Definition~\ref{s:preliminaries}).
So, in Potential games, we have the same convergence properties as in single-function optimization. 
The convergence into the local minima follows Newton's convergence proofs for non-convex functions \cite{paternain2019newton}.

\subsection{Newton's Method for Hamiltonian Games}
Hamiltonian games are characterized by a Hamiltonian function $\Ham : \mathbb{R}^{n \times d} \rightarrow \mathbb{R}$. In these games, the gradient descent does not converge to a stable fixed point but causes cyclical behavior. Instead, the gradient descent on the Hamiltonian function converges to a Nash equilibrium. Figure~\ref{fig:dynamics_ham} shows the dynamics of gradient descent w.r.t. $\xi$ and $\nabla \Ham$ on a Hamiltonian game: the figure points out that a gradient descent on $\xi$ cycles.

\begin{exmp}
Take a two-player bilinear game with agents with parameters $\vtheta_1$ and $\vtheta_2$ minimizing respectively $f(\vtheta): \mathbb{R}^{n \times d} \rightarrow \mathbb{R}$ and $g(\vtheta): \mathbb{R}^{n \times d} \rightarrow \mathbb{R}$ respectively. A point in this class of games is a Nash Equilibrium if $\xi(\vtheta) = 0$, i.e., $\nabla_{\vtheta_1} f = 0$ and $\nabla_{\vtheta_2} g = 0$, because $\nabla^2_{\vtheta_1}  f = 0$ and  $\nabla^2_{\vtheta_2}  g = 0$. Considering this, the Nash Equilibrium can be calculated in closed form (if the inverse exists\footnote{If the inverse does not exist and a Nash Equilibrium exists, the system is indeterterminate. Using the Moore-Penrose inverse we find an approximate Nash Equilibrium (the one with the smallest Euclidean norm).} 
) by setting the gradient equal to zero:
\begin{align*}
\begin{bmatrix}
\nabla_{\vtheta_1} f \\
\nabla_{\vtheta_2} g
\end{bmatrix}
+ 
\begin{bmatrix}
0 & \nabla_{\vtheta_1, \vtheta_2} f \\
\nabla_{\vtheta_2, \vtheta_1} g & 0
\end{bmatrix} 
\begin{bmatrix}
{\vtheta_1}  \\
{\vtheta_2} 
\end{bmatrix} = 0 \\
\begin{bmatrix}
{\vtheta_1}  \\
{\vtheta_2} 
\end{bmatrix} =
- \begin{bmatrix}
0 & \nabla_{\vtheta_1, \vtheta_2} f \\
\nabla_{\vtheta_2, \vtheta_1} g & 0
\end{bmatrix}^{-1}
 \begin{bmatrix}
\nabla_{\vtheta_1} f \\
\nabla_{\vtheta_2} g
\end{bmatrix}.
\end{align*}
\end{exmp}
The example above provides the intuition that the solution to quadratic Hamiltonian games is achieved by the following update rule (in Hamiltonian games $S = 0$):
\begin{equation}
\label{eq:hamiltonian_newton}
\vtheta_{t+1} = \vtheta_t - A(\vtheta_t)^{-1} \xi(\vtheta_t).
\end{equation}
In the following theorem, we state that even in this class of games the convergence to a local Nash Equilibrium (using the above update) is quadratic (see Figure~\ref{fig:dynamics_ham}). 
The proof is in Appendix A.

\begin{restatable}[]{thr}{hamiltonian}
\label{thm:hamiltonian}
Suppose that $\xi$ and $A$ are twice continuous differentiable and that $A$ is invertible in the (local) Nash Equilibrium $\vtheta^*$.  Then, there exists $\epsilon > 0$ such that iterations starting from any point in the ball $\vtheta_0 \in B(\vtheta^*, \epsilon)$ with center $\vtheta^*$ and ray $\epsilon$ converge to $\vtheta^*$. Furthermore, the convergence rate is quadratic.
\end{restatable}
      \begin{algorithm}[t]
    \caption{\algname}\label{alg:alg}
    \small
    \textbf{input}: discounted returns $ \mathcal V = \{V_i\}_{i=1}^n$, PT inverse parameter $m$ \\
    \textbf{output}: update rule
    \begin{algorithmic}
    	\STATE  Compute $\xi, \Hes, S, A, S^{-1}_m$
    	\IF {$\cos{\nu_S} \ge 0$}
  \STATE \textbf{if} $\cos{\nu_S} \ge \cos{\nu_A}$ \textbf{then} \eqref{eq:potential_newton} \textbf{else} \eqref{eq:hamiltonian_newton}
  	\ELSE
  	\STATE \textbf{if} $\cos{\nu_S} \le \cos{\nu_A}$ \textbf{then} \eqref{eq:potential_newton} \textbf{else} \eqref{eq:hamiltonian_newton}
    		\ENDIF
    \end{algorithmic}
 \end{algorithm}
\subsection{Newton's Method for General Games}
In general games, it is not yet known whether and how the system's dynamics can be reduced to a single function as for Potential and Hamiltonian games.
Thus, finding a Newton-based update is more challenging: if we apply Newton's Method with the Jacobian PT-transformation, we can alter the Hamiltonian dynamics of the game. Instead, applying the inverse of the Jacobian as in the Hamiltonian games can lead to local maxima. In this section, we show how to build a Newton-based learning rule that guarantees desiderata similar to those considered in~\citep{balduzzi2018mechanics}: the update rule has to be compatible (D1) with Potential dynamics if the game is a Potential game, with (D2) Hamiltonian dynamics if the game is a Hamiltonian game and has to be (D3) attracted by symmetric stable fixed points and (D4) repelled by unstable ones. By compatible we mean that given two vectors $u, v$ then $u^Tv > 0$.

The algorithm that we propose (see Algorithm~\ref{alg:alg}) chooses the update to perform between the two updates in~\eqref{eq:potential_newton} and~\eqref{eq:hamiltonian_newton}. The choice is based on the angles between the gradient of the Hamiltonian function $\Ham$ and the two candidate updates' directions.
In particular, we compute 
\begin{equation*}
\cos{\nu_S} = \frac{(S_m^{-1}\xi)^T\nabla\Ham}{\nrm{S_m^{-1}\xi}\nrm{\nabla\Ham}},~ \cos{\nu_A} = \frac{(A^{-1}\xi)^T\nabla\Ham}{\nrm{A^{-1}\xi}\nrm{\nabla\Ham}}.
\end{equation*}
When the cosine is positive, the update rule follows a direction that reduces the value of the Hamiltonian function (i.e., reduces gradient norm), otherwise the update rule points in an increasing direction of the Hamiltonian function. 
Notably, there is a connection between the positive/negative definiteness of $\Ham$ and the sign of $\cos{\nu_S}$.

\begin{restatable}[]{lemma}{cosine}
\label{L:cosine}
Given the Jacobian $\Hes = S + A$ and the simultaneous gradient $\xi$, if $S \succeq 0$ then $\cos{\nu_S} \ge 0$; instead if $S \prec 0$ then $\cos{\nu_S} < 0$.
\end{restatable}
The idea of \algname~  is to use the sign of $\cos{\nu_S}$ to decide whether to move in a direction that reduces the Hamiltonian function (aiming at converging to a stable fixed point) or not (aiming at getting away from unstable points). In case $\cos{\nu_S}$ is positive, the algorithm chooses the update rule with the largest cosine value (i.e., which minimizes the angle with $\nabla\Ham$), otherwise, \algname~ tries to point in the opposite direction by taking the update rule that minimizes the cosine. In the following theorem, we show that the update performed by \algname~ satisfies the desiderata described above.
\begin{restatable}[]{thr}{reqth}
\label{th:reqtheorem}
The \algname~ update rule satisfies requirements (D1), (D2), (D3), and (D4).
\end{restatable}
\begin{proof}{\textit{(Sketch)}}
The requirement (D1) is satisfied if $(S^{-1}_m\xi)^T\nabla\phi$ and $(A^{-1} \xi)^T \nabla\phi$ are nonnegative and $\mathcal{G}$ is a Potential game; in this case the update rule is $S^{-1}_m \xi$ because $A = 0$. We have that $\nabla \phi = \xi$ and $\xi^T S^{-1}_m \xi \ge 0$ as said before. For requirement (D2) we can make similar considerations: in this case we have to show that $(S^{-1}_m\xi)^T\nabla\mathcal{H}$ and $(A^{-1} \xi)^T \nabla\mathcal H$ are nonnegative when $\mathcal{G}$ is a Hamiltonian game, that it is equal to say: $(A^{-1} \xi)^T A^T \xi = \nrm{\xi}^2$. Finally, the fulfillment of desiderata (D3) and (D4) is a consequence of Lemma~\ref{L:cosine}.
\end{proof}
Given the results from Theorem~\ref{th:reqtheorem} and Lemma~\ref{L:cosine}, we can argue that if $\xi$ points at a stable fixed point then \algname~ points also to the stable fixed point otherwise if $\xi$ points away from the fixed point also \algname~ points away from it.\\
Then we prove that \algname~ converges only to fixed points and, under some conditions, it converges locally to symmetric stable fixed points.
\begin{restatable}[]{lemma}{fixedpoint}
\label{thm:fixedpoint}
If \algname~ converges to a $\vtheta^*$ then $\xi(\vtheta^*) = \mathbf{0}$.
\end{restatable}
\begin{restatable}[]{thr}{generalsum}
\label{thm:generalsum}
Suppose $\vtheta^*$ is a stable fixed point, and suppose $A, S, \Hes$ are bounded and Lipschistz continuous with modulus respectively $M_A, M_S, M_{\Hes}$ in the region of attraction of the stable fixed point $\xi(\vtheta^*)$.  Furthermore assume that $\norm{A^{-1}} \le N_A$ and  $\norm{S^{-1}} \le N_S$. Then there exists $\epsilon > 0$ such that the iterations starting from any point $\vtheta_0 \in B(\vtheta^*, \epsilon)$ converge to $\vtheta^*$.  
\end{restatable}
More details about the proofs are given in Appendix A.\\
\begin{remark} 
In this paper, we focus our attention on convergence towards symmetric stable fixed points. However, some Nash Equilibria are not symmetric stable fixed points. On the other hand, symmetric stable fixed points are an interesting solution concept since in two-player zero-sum games (e.g., GANs \cite{goodfellow2014generative}), all local Nash Equilibria are also symmetric stable fixed points~\cite{balduzzi2018mechanics}.
\end{remark}

\section{Learning via Policy Gradient in Stochastic Games}\label{s:learning}
Usually, agents do not have access to the full gradient or the Jacobian, and we need to estimate them.
We define a $T$-episodic trajectory as $\tau = (x_0, \mathbf{u}_0, C_0^1, \dots, C_0^n, \dots x_T, \mathbf{u}_T, C_T^1, \dots, C_T^n )$.
Following policy gradient derivation \cite{peters2008reinforcement}, $\nabla_{\theta_i} V_i$ can be estimated by:\footnote{With $\widehat{\nabla}^M_{\vtheta_i}$ we intend the estimator of $\nabla_{\vtheta_i}$ over $M$ samples.}
\small{
\begin{align*}
\widehat{\nabla}^M_{\vtheta_i} V_i = \frac{1}{M} \sum_{m=1}^M \sum_{t=0}^T \gamma^t C_{m,t}^i \sum_{t^{'} = 0}^t \nabla \log \pi_{\vtheta_i} (x_{m,t'}, u^{i}_{m,t'}) .
\end{align*}}
\normalsize
Then, to estimate the Jacobian $\Hes$, we have to compute the second-order gradient $\nabla_{\vtheta_k} \nabla_{\vtheta_j} V_i$, with $ 1 \le k \le n$, $1 \le j \le n$. If $k \ne j$ we derive~\cite{foerster2018learning} the second-order gradient, exploiting the independence of agents' policies:
\small{
\begin{align*}
&\widehat{\nabla}^M_{\vtheta_k, \vtheta_j}  V_i = \frac{1}{M} \sum_{m=1}^M \sum_{t=0}^T \gamma^t C_{m,t}^i \\ & \quad \times\sum_{t^{'} = 0}^t \left(\nabla \log \pi_{\vtheta_k} (x_{m,t'}, u^{k}_{m,t'} ) \right)  \left(\nabla \log \pi_{\vtheta_j} (x_{m,t'}, u^{j}_{m,t'} ) \right)^T .
\end{align*}}
\normalsize
 When $k = j$ we are evaluating the second-order gradient of $\pi_{\vtheta_j}$. To evaluate this part we derive the second-order gradient as done for the single-agent case~\cite{pmlr-v97-shen19d}, with  $g^i_m(\vtheta_k) =  \sum_{t = 0}^T  C_{t}^i \sum_{t^{'} = 0}^t\nabla \log \pi_{\vtheta_k} (x_{m,t'}, u_{m,t'}^i)) $: \footnote{ All derivations are reported in Appendix.}
\small{
\begin{align*}
\widehat{\nabla}^{2,M}_{\vtheta_k} V_i = \frac{1}{M} \sum_{m=1}^M \nabla g^{i}_{m}(\vtheta_k)
\nabla \log \pi_{\vtheta_k}(\tau_m)^T +  \nabla^2 g^i_m(\vtheta_k).
\end{align*}}
\normalsize
\section{Experiments}\label{s:experiments}
\begin{figure*}[h!]
\centering
\includegraphics[width=.8\textwidth]{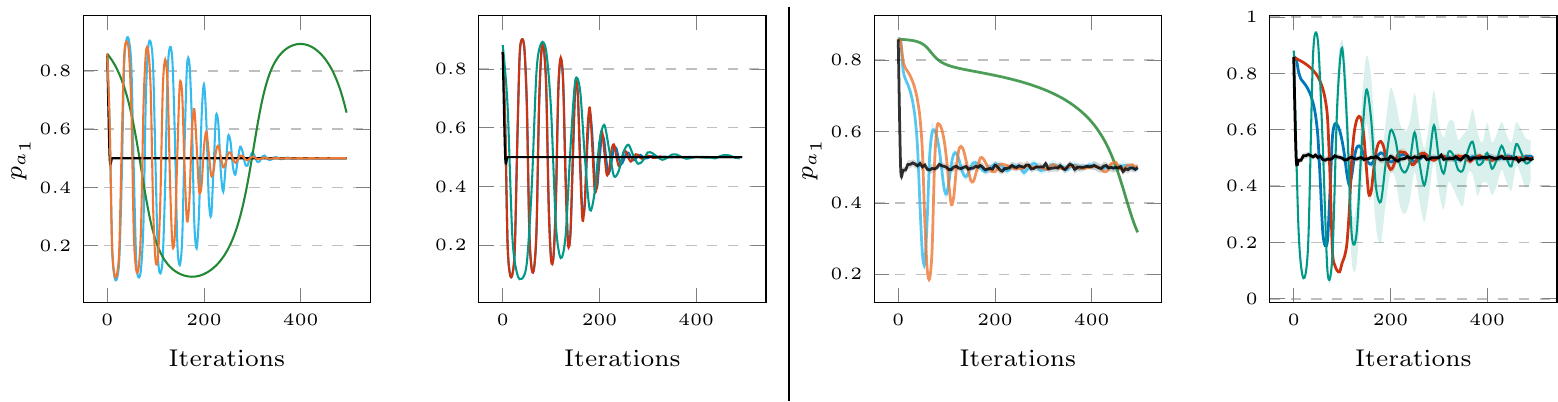}
\label{fig:MP}
\centering
\centering
\includegraphics[width=.8\textwidth]{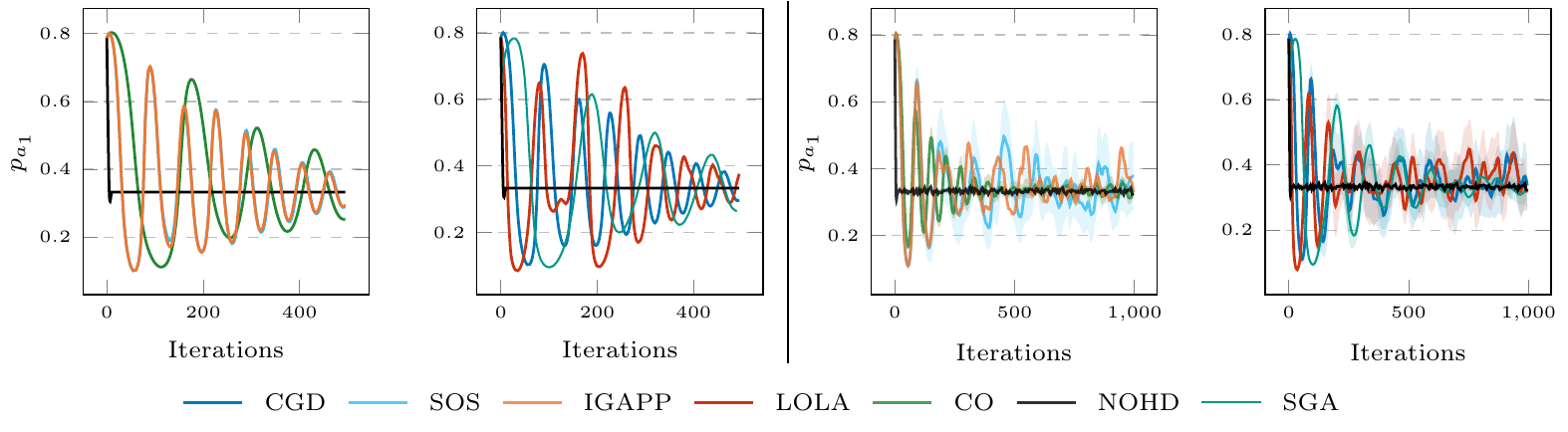}
\caption{Matching Pennies and Rock Paper Scissors (on the rows): probabilities of agent $1$ to perform first action with a Boltzmann parametrization of the policies. On the left: plots with exact gradients. On the right: plots with estimated gradients. Every algorithm with its best learning rate between $0.1, 0.5, 1.0$. c.i.~95. \%}
\label{fig:RPS}
\end{figure*}


This section is devoted to the experimental evaluation of \algname. The proposed algorithm is compared with Consensus Optimization (CO) \citep{mescheder2017numerics}, Stable Opponent Shaping (SOS) \citep{letcher2018stable}, Learning with Opponent-Learning Awareness (LOLA) \citep{foerster2018learning}, Competitive Gradient Descent (CGD) \cite{schafer2019competitive}, Iterated Gradient Ascent Policy Prediction (IGA-PP) \cite{zhang2010multi} and Symplectic Gradient Adjustment (SGA) \cite{balduzzi2018mechanics}.

\subsection{Matrix Games}

We consider two matrix games: two-agent two-action Matching Pennies (MP) and two-agent three-action Rock Paper Scissors (RPS) (games' rewards are reported in Appendix B). Considering a linear parameterization of the agents' policies, MP and RPS are Hamiltonian games \footnote{In Appendix we report results also for Dilemma game}. The Nash equilibria of MP, and RPS are respectively $0.5$, and $0.333$ regarding the probability of taking action $1$. 
For the first experiment, we used a Boltzmann parametrization for the agents' policies and exact computation of gradients and Jacobian. In this setting, games lose their Hamiltonian or Potential property, making the experiment more interesting and the behavior of \algname~ not trivial. The results are shown in Figure~\ref{fig:RPS} (left side). For each game, we perform experiments with learning rates $0.1, 0.5, 1.0$. In the plots are reported only the best performance for each algorithm. In Matching Pennies we initialize probabilities to $[0.86, 0.14]$ for the first agent and to $[0.14, 0.86]$ for the second agent; instead in Rock Paper Scissors to $[0.66, 0.24, 0.1]$. \footnote{For readability, we show only a starting point. Appendix B contains the results for different starting points and more experiments.} The figure shows that each algorithm is able to converge to the Nash equilibrium (in MP, CO converges with a learning rate of $0.1$ and takes more than $1000$ iterations. The plot is reported in Appendix). In Table~\ref{tab:ratio} we reported the ratio between the number of steps each algorithm takes to converge and the maximum number of steps in which the slowest algorithm converges. For this simulation, we sampled $50$ random initializations of the parameters from a normal distribution with zero mean and standard deviation $0.5$. 
Table~\ref{tab:ratio} shows that \algname~ significantly outperforms other algorithms even when starting from random initial probabilities. \\
In the second experiment, the gradients and the Jacobian are estimated from samples. The starting probabilities are the same as in the previous experiment. We performed $20$ runs for each setting. In each iteration, we sampled $300$ trajectories of length $1$.  Figures~\ref{fig:RPS} (right side) show that \algname~ also in this experiment converges to the equilibrium in less than $100$ iterations. Instead, the other algorithms exhibit oscillatory behaviors. 
\begin{table}[h!]
\centering
\resizebox{.47\textwidth}{!}{
\begin{tabular}{c|ccccccc}
\hline
  & \algname & CGD    & LOLA  &IGA-PP & CO & SOS & SGA  \\
  \hline
  MP & $\textbf{0.49}$ &$0.84$ & $1.00$  & $0.99$ & $0.99$& $0.99$& $ 0.99$ \\
  RPS &$\textbf{0.38} $ &$0.97$&  $0.88 $& $ 1.00$& $ 0.81$& $ 0.80$ & $0.96$ \\
\hline
\end{tabular}}
\caption{Ratio between the mean convergence steps to Nash Equilibrium and the maximum mean convergence steps. $50$ runs, sampling from a normal distribution $\mathcal{N}(0,0.5^2)$.}
\label{tab:ratio}
\end{table}
\begin{table}[h!]
\centering
\resizebox{.48\textwidth}{!}{
\begin{tabular}{c|ccccccc}
\hline 
   $|\vtheta|$ & NOHD    & IGAPP   & LOLA    & SOS     & SGA     & CGD     & CO \\ \hline
4   & 0.7205 & 0.6979 & 0.6983 & 0.7186 & 0.7302 & 0.7265 & 0.7006   \\
16  & 0.7898 & 0.7787 & 0.7735 & 0.7906 & 0.8358 & 0.8051 & 0.7758   \\
36  & 1.1416 & 1.0625 & 1.0874 & 1.0705 & 1.1992 & 1.1444 & 1.1066   \\
64  & 1.9486 & 1.6342 & 1.6162 & 1.6735 & 1.9555 & 1.8955 & 1.8850   \\
100 & 3.4070 & 2.7734 & 2.6905 & 2.8169 & 3.4799 & 3.3126 & 3.2977   \\
144 & 5.9191 & 4.6260 & 4.4351 & 4.8438 & 6.5164 & 5.8440 & 5.7876  \\
\hline
\end{tabular}}
\caption{Computation time of one learning update of each algorithm with increasing parameter space size. 20 runs. }
\label{tab:computational}
\end{table}

\begin{figure*}[t!]
\centering
\includegraphics[width=0.85\textwidth]{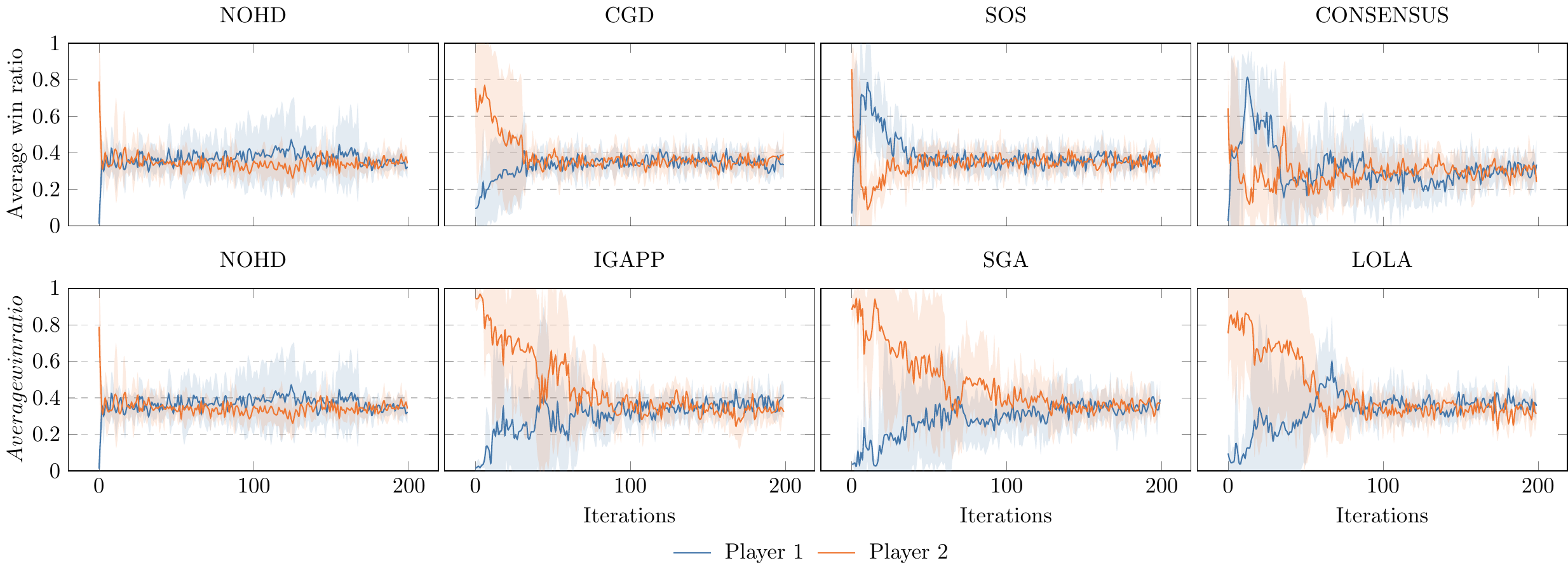}
\\
\includegraphics[width=0.85\textwidth]{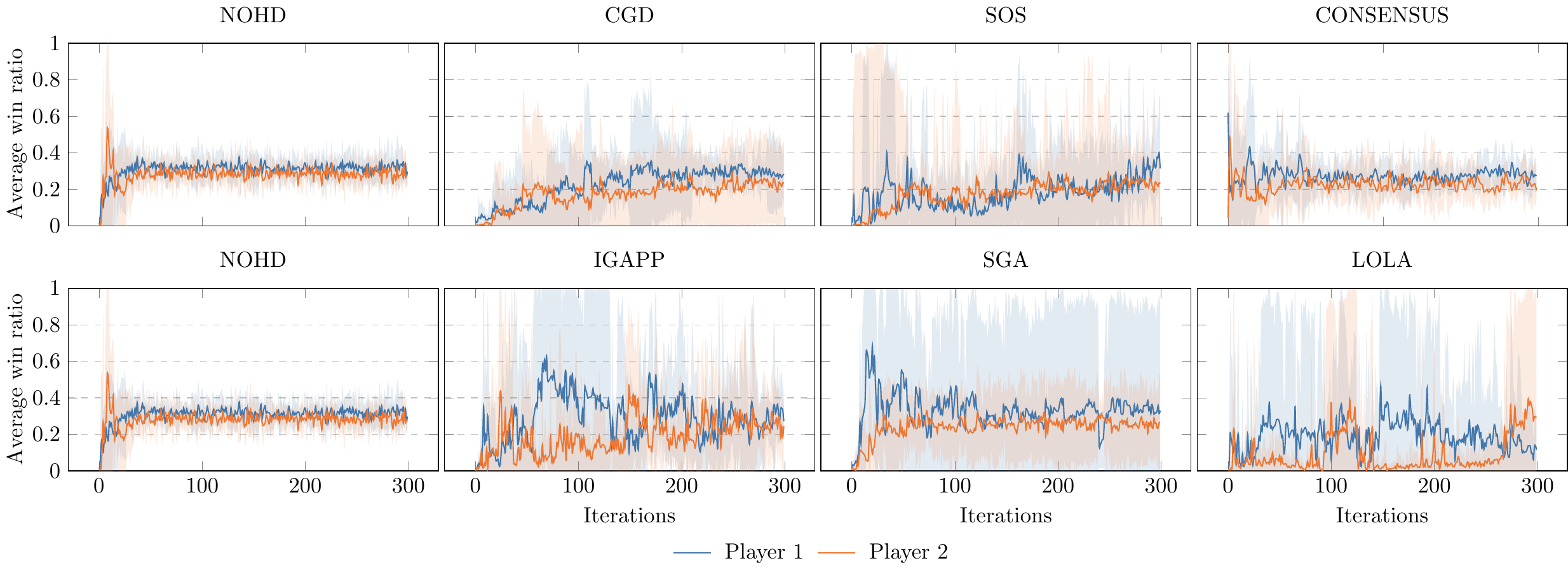}
\caption{Average wins for player $1$ and player $2$ in two gridworld environments. $10$ runs, c.i.~98\%}
\label{fig:gridworld}
\end{figure*}
\subsection{Continuous Gridworlds}
This experiment aims at evaluating the performance of \algname~ in two continuous Gridworld environments. The first gridworld is the continuous version of the second gridworld proposed in \cite{hu2003nash}: the two agents are initialized in the two opposite lower corners and have to reach the same goal; when one of the two agents reaches the goal, the game ends, and this agent gets a positive reward. Each agent has to keep a distance of no less than $0.5$ with the other agent and, if they decide to move to the same region, they cannot perform the action. Also, in the second gridworld the agents are initialized in the two lower corners and have to reach the same goal, but they have to reach the goal with a ball. An agent can take the opponent's ball if their distance is less than $0.5$; the ball is randomly given to one of the two agents at the beginning of each episode. The agents' policies are Gaussian policies, linear in a set of respectively $72$ and $68$ radial basis functions, which generate the $\nu$ angle for the step's direction.~\footnote{More information about the experiments are given in Appendix.} For each experiment, we perform $10$ runs with random initialization.
In Figure~\ref{fig:gridworld} we compare the performance of \algname~ with CO, IGA-PP, LOLA, SOS, and CGD in the two gridworlds. The figure shows the mean average win of the two players at each learning iteration of the algorithms. In figure~\ref{fig:gridworld}, at the top, we can see how NOHD outperforms the other algorithms in the first gridworld, converging in less than $30$ steps to the equilibrium. At the bottom, results for the second gridworld are shown: \algname~ converges quickly than the other algorithms, but Consensus has comparable performance.

\subsection{Computational Time}
In Table~\ref{tab:computational} we report the computation time of an update of each algorithm with increasing policy parameter sizes $|\vtheta|$, from $4$ to $144$.  As we can see, the computation time of \algname~ is comparable to that of the other algorithms.

\section{Conclusions}
Although multi-agent reinforcement learning has achieved promising results in recent years, few algorithms consider the dynamics of the system. In this paper, we have shown how to apply Newton-based methods in the multi-agent setting. The paper's first contribution is to propose a method to adapt Newton's optimization to two simple game classes: Potential games and Hamiltonian games. Next, we propose a new algorithm \algname~ that applies a Newton-based optimization to general games. The algorithm, such as SGA, SOS, and CGD, considers that agents can also act against their own interests to achieve a balance as quickly as possible. We also show that \algname~ avoids unstable equilibria and is attracted to stable ones. We then show how the algorithm outperforms some baselines in matrix games with parametric Boltzmann policies. Furthermore, the algorithm manages to learn good policies in two continuous gridworld environments. In future work, we will try to extend the algorithm \algname~ to settings where the agent cannot know the other agents' policies and cost functions.

%

\bibliographystyle{aaai}
\bibliography{example_paper.bib}

\appendix
\onecolumn
\section{Proofs and derivation}
\label{app}
In this appendix, we report the proofs and derivations of the results presented in the main paper.
\subsection{Proofs of Section 4}
\label{app:proof4}

\hamiltonian*
\begin{proof}
The proof is the standard proof of convergence for Newton's methods. We report here the steps for completeness.
Since $\xi$ is twice differentiable, its Taylor series expansion in $\vtheta_0$ is:
\begin{equation}
\xi(\vtheta) - \xi(\vtheta_0) - A(\vtheta - \vtheta_0) = O(\nrm{\vtheta - \vtheta_0})^2.
\end{equation}
Because $A$ is twice continuous differentiable then $\exists \epsilon_1, M > 0$ s.t. if we take $\vtheta_0, \vtheta \in B(\vtheta^*, \epsilon_1)$ then we have that:
\begin{equation}
\label{lip}
\nrm{\xi(\vtheta) - \xi(\vtheta_0) - A(\vtheta  - \vtheta_0)} \le M \nrm{\vtheta - \vtheta_0}^2.
\end{equation}
Since $A$ is twice continuous differentiable and for assumption it is invertible in $\vtheta^*$ then there exists $\epsilon_2, N$ s.t. $ \forall \vtheta \in B(\vtheta^*, \epsilon_2$ $A^{-1}$ exists and $\nrm{A^{-1}(\vtheta)} \le N$ (see lemma 5.3 in \cite{chong2004introduction}).

Let $\epsilon = \min(\epsilon_1, \epsilon_2)$.
Now, we substitute $\vtheta$ with $\vtheta^*$ in \ref{lip} and we use the assumption that $\xi(\vtheta^*) = 0$:
\begin{equation}
\nrm{A(\vtheta_0  - \vtheta^*)- \xi(\vtheta_0)} \le M \nrm{\vtheta - \vtheta_0}^2.
\end{equation}
If we use the update rule and we take $\vtheta_1 \in B(\vtheta^*, \epsilon)$ we have that:
\begin{align}
\nrm{\vtheta_1 - \vtheta^*} &= \nrm{\vtheta_0 - \vtheta^* - A^{-1}(\vtheta_0) \xi(\vtheta_0)} \nonumber \\
&= \nrm{A(\vtheta_0)^{-1} (A(\vtheta_0) (\vtheta_0 - \vtheta^*) - \xi(\vtheta_0))} \nonumber \\
&\le \nrm{A(\vtheta_0)^{-1}} \nrm{(A(\vtheta_0) (\vtheta_0 - \vtheta^*) - \xi(\vtheta_0))},
\end{align}
where in the last step we used the triangular inequality. If we used the inequalities \ref{lip} we have that:
\begin{equation}
\nrm{\vtheta_1 - \vtheta^*}  \le MN \nrm{\vtheta_0 - \vtheta^*}^2.
\end{equation}
If we suppose that $\nrm{\vtheta_0 - \vtheta^*} \le \frac{\alpha}{MN}$ with $\alpha \in (0,1)$, then:
\begin{equation}
\nrm{\vtheta_1 - \vtheta^*}  \le \alpha \nrm{\vtheta_0 - \vtheta^*}^2.
\end{equation}
Then by induction we obtain that:
\begin{equation}
\nrm{\vtheta_{k+1} - \vtheta^*}  \le \alpha \nrm{\vtheta_{k} - \vtheta^*}^2.
\end{equation}
Hence $\lim_{k \rightarrow \infty} \nrm{\vtheta_k - \vtheta^*} = 0$ and therefore the sequence $\vtheta_k$ converges to $\vtheta^*$ if we take $\epsilon \le \frac{\alpha}{MN}$, and the order of convergence is at least $2$. 
\end{proof}

\cosine*
\begin{proof}
We know that $\cos \nu_S =  \frac{(S_m^{-1}\xi)^T\nabla\Hes}{\nrm{S_m^{-1}\xi}\nrm{\nabla\Hes}}$; then the sign of $\cos \nu_S$ depends on $(S_m^{-1}\xi)^T\nabla\Hes$. Suppose that $S \succeq 0$. We show that if  $S \succeq 0$ then $S^{-1}_m (S^T + A^T) \succeq 0$:
\begin{align*}
S^{-1}_m (S+A^T) = S_m^{-\frac{1}{2}}S_m^{-\frac{1}{2}}(S+A^T) =  S_m^{-\frac{1}{2}}(S_m^{-\frac{1}{2}}(S+A^T)S_m^{-\frac{1}{2}})S_m^{\frac{1}{2}}.
\end{align*}
We use the fact that $S^{-1}_m$ is positive definite for construction. So there exists a unique square root matrix $S_m^{-1/2}$ that is symmetric. Then the matrix $S^{-1}_m (S+A^T)$ is similar to $S_m^{-1/2}(S+A^T)S_m^{-1/2}$. For every vector $u \in \Real^{n \times d}$:
\begin{align*}
u^T S_m^{-1/2}(S+A^T)S_m^{-1/2} u = z (S + A^T) z \ge 0,
\end{align*}
where $z = u^T S_m^{-1/2} = S_m^{-1/2} u$ because $S_m^{-1/2}$ is symmetric.
Using the same reasoning it is shown that if $S \prec 0$ then $S^{-1}_m (S^T + A^T) \prec 0$.
\end{proof}

\reqth*
\begin{proof}
\textbf{D1} \algname~has to be compatible with Potential game dynamics if the game is a Potential game: $\xi^T S^{-1}_m \nabla \phi > 0$ and $\xi (A^{-1})^T  \nabla \phi > 0$. We notice that $\nabla \phi = \xi$ and that $A = 0$ because the game is a Potential game. $\xi^T S^{-1}_m \xi > 0$ for every $\xi \ne \mathbf{0}$ since $S^{-1}_m$ is positive definite for construction.

\textbf{D2} \algname~has to be compatible with Hamiltonian game dynamics if the game is a Hamiltonian game: $\xi S^{-1}_m \nabla \Ham > 0$ and $\xi (A^{-1})^T  \nabla \Ham > 0$. We know that $\nabla \Ham = (S^T+A^T) \xi$ and that $S = 0$ because the game is a Hamiltonian game. Then $\xi^T (A^{-1})^T  (A^T) \xi =  \nrm{\xi}^2$.

\textbf{D3} \algname~has to be attracted to symmetric stable fixed points. It means that if $S+A$ is positive definite, and so, $\xi^T (S+A) \xi \ge 0$. Then $S \succeq 0$. From Lemma~\ref{L:cosine} we know that also $\xi^T S^{-1}_m (S+A^T) \xi \ge 0$ so $\cos_{\nu_S} \ge 0$. The update rule take the $\max(\cos_{\nu_S}, \cos_{\nu_A})$ that from the previous consideration is always positive. 

\textbf{D4} \algname~has to be repelled by unstable symmetric fixed points. It means that if $S+A$ is negative definite, and so, $\xi^T (S+A) \xi \ge 0$. Then $S \prec 0$.From Lemma~\ref{L:cosine} we know that also $\xi^T S^{-1}_m (S+A^T) \xi \ge 0$ so $\cos_{\nu_S} < 0$. The update rule take the $\min(\cos_{\nu_S}, \cos_{\nu_A})$ that from the previous consideration is always strict negative.
\end{proof}

\fixedpoint*
\begin{proof}
Suppose that $\vtheta^*$ is not a fixed point, so $\xi(\vtheta^*) \ne 0$. The process is stopped in $\vtheta^*$ if and only if $S^{-1}_m \xi = \mathbf{0}$ and $A^{-1} \xi = \mathbf{0}$, because if $\xi^T A^{-1}_m = \mathbf{0}$ and $\xi^T S^{-1}_m \ne \mathbf{0}$ we always take $- S^{-1}_m \xi$ as update since $\nrm{\cos_{\nu_S}} \ge \nrm{\cos_A}$. $\xi^T S^{-1}_m = \mathbf{0}$ only if $\xi = \mathbf{0}$ because $S^{-1}_m$ is positive definite by construction, then we contradict the hypothesis.

We have to mention that with this lemma we prove only that the convergence points of the game are fixed points, i.e. $\xi = \mathbf{0}$.
\end{proof}

\generalsum*
\begin{proof}
Since $\xi$ is differentiable by assumption we can write:
\begin{equation*}
\xi(\vtheta)-\xi(\vtheta^*) = \int_0^1[\Hes(\vtheta_n + t(\vtheta^*-\vtheta_n))] (\vtheta^* - \vtheta_n) dt.
\end{equation*}
So we have:
\begin{align*}
\vtheta_{1} - \vtheta^* &= \vtheta_{0} - S(x)^{-1} \xi(\vtheta) - \vtheta^* \\
&= \vtheta_{0} + S(x)^{-1} (\xi(\vtheta)-\xi(\vtheta^*)) - \vtheta^*\\
&=  \vtheta_{0} - \vtheta^* S(x)^{-1}\int_0^1\left(\Hes(\vtheta_0 + t(\vtheta^*-\vtheta_0))\right) (\vtheta^* - \vtheta_0) dt \\
&= S(\vtheta)^{-1} \int_0^1 \left(\Hes(\vtheta_0 + t(\vtheta^*-\vtheta_0))-S(\vtheta_0)\right) (\vtheta^* - \vtheta_0) dt.
\end{align*}
Taking the norm and supposing that the current update is with $S(\vtheta)^{-1}$
\begin{align*}
\norm{\vtheta_{1} - \vtheta^*} &\le \norm{S(\vtheta)^{-1}} \int_0^1 \norm{\Hes(\vtheta_0 + t(\vtheta^*-\vtheta_0))-S(\vtheta_0)} \norm{\vtheta^* - \vtheta_n} dt \\
&= \norm{S(\vtheta)^{-1}} \int_0^1 \norm{\Hes(\vtheta_0 + t(\vtheta^*-\vtheta_0))-\Hes(\vtheta_0)+A(\vtheta_0)} \norm{\vtheta^* - \vtheta_0} dt \\
&\le \norm{S(\vtheta)^{-1}} \left(\int_0^1\norm{\Hes(\vtheta_0 + t(\vtheta^*-\vtheta_0))-\Hes(\vtheta_0)} \norm{\vtheta^* - \vtheta_0} dt +\right. \\
&\left.\qquad\qquad\qquad +  \int_0^1\norm{[A(\vtheta_0)]} \norm{\vtheta^* - \vtheta_0} dt\right) \\
&\le  \norm{S(\vtheta)^{-1}} \norm{\vtheta^* - \vtheta_0} \left(\int_0^1 L t \norm{\vtheta^* - \vtheta_0} dt+ \int_0^1 M_A dt\right) \\
&\le N_S L \norm{\vtheta^* - \vtheta_0}^2 + N_A M_A \norm{\vtheta^* - \vtheta_0} \le (N_S L + N_A M_A) \norm{\vtheta^* - \vtheta_0}.
\end{align*}
If we suppose that $ \norm{\vtheta^* - \vtheta_0} \le \frac{\alpha}{ (N_S L + N_A M_A)}$ with $\alpha \in (0,1)$, then:
\begin{equation*}
\norm{\vtheta^* - \vtheta_{1}} \le \alpha \norm{\vtheta^* - \vtheta_0}.
\end{equation*}
Then by induction:
\begin{equation*}
\norm{\vtheta^* - \vtheta_{n}} \le \alpha \norm{\vtheta^* - \vtheta_{n-1}}.
\end{equation*}
Hence, $\lim_{n \rightarrow \infty} \norm{\vtheta^* - \vtheta_{n}} = 0$, so if we take $\epsilon \le \frac{\alpha}{ \max{(N_S L + N_A M_A, N_A L + N_S M_S)}}$ the sequence converges. 

\end{proof}

\newpage
\section{Experimental results} 
\label{app:experiment}
In this appendix, we report some additional experimental results on the two games Matching Pennies and Rock--Paper--Scissors. We compare \algname~against $6$ baselines (Consensus~\cite{mescheder2017numerics}, LOLA~\cite{foerster2018learning}, IGAPP~\citep{zhang2010multi}, SOS~\cite{letcher2018stable}, CGD~\citep{schafer2019competitive}), SGA~\cite{balduzzi2018mechanics}. We settled the hyperparameter $a, b$ of SOS as in the original paper $a=0.5$ and $b=0.1$ (even if the experiments in the original paper are on another game). The parameter $m$ of the PT-inverse is settled to $0.03$ in all the experiments. We conduct experiments with a linear parametrization of the policy and a Boltzmann parametrization of the policy. In the Boltzmann experiments we show results with exact gradients and estimated gradients. We estimated the gradients with batch size equals to $300$ and horizon $1$.
Below we report the payoff matrices of Matching Pennies (Table~\ref{MPpayoff}), Rock--Paper--Scissors (Table~\ref{RPSpayoff}), and Dilemma (Table~\ref{Dilemmapayoff}).

\begin{table*}[t]
\parbox{.5\linewidth}{
\centering
\begin{tabular}{|c|cc|cc|}
\hline
     & \multicolumn{2}{c|}{Head} & \multicolumn{2}{c|}{Tail} \\ \hline
Head & 1           & -1          & -1          & 1           \\ \hline
Tail & -1          & 1           & 1           & -1          \\ \hline
\end{tabular}
\caption{Matching pennies payoffs.}
\label{MPpayoff}
}
\parbox{.5\linewidth}{
\centering
\begin{tabular}{|c|cc|cc|}
\hline
     & \multicolumn{2}{c|}{Head} & \multicolumn{2}{c|}{Tail} \\ \hline
Head & -1           & -1          & -3          & 0          \\ \hline
Tail & 0          & -3           & -2          & -2          \\ \hline
\end{tabular}
\caption{Dilemma payoffs.}
\label{Dilemmapayoff}
}
\end{table*}
\begin{table*}[t]
\centering
\centering
\begin{tabular}{|c|cc|cc|cc|}
\hline
         & \multicolumn{2}{c|}{Rock} & \multicolumn{2}{c|}{Paper} & \multicolumn{2}{l|}{Scissors} \\ \hline
Rock     & 0           & 0           & -1           & -1          & 1             & -1            \\ \hline
Paper    & 1           & -1          & 0            & 0           & -1            & 1             \\ \hline
Scissors & -1          & 1           & 1            & -1          & 0             & 0             \\ \hline
\end{tabular}
\caption{RPS payoffs.}
\label{RPSpayoff}
\end{table*}
In Figure~\ref{consensusmore} we report the results for Consensus with $2000$ iterations and learning rate $0.1$ in order to show that the algorithm converges.
\begin{figure}[t!]
\centering
\includegraphics[scale=1]{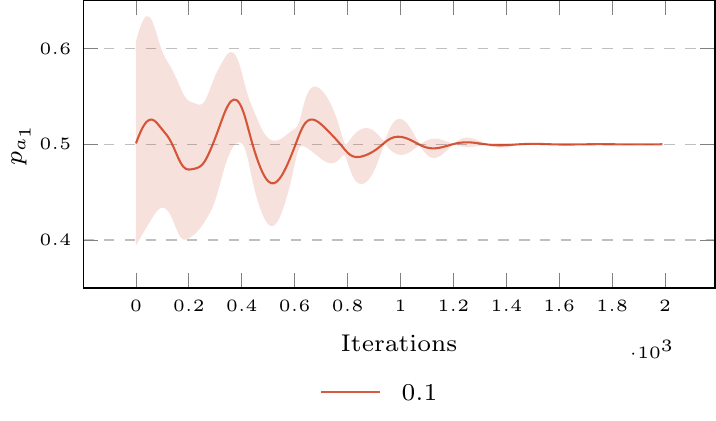}
\caption{Agent 1's probabily to perform action $1$ in Matching Pennies with Consensus. Exact gradients. $20$ repetitions. Learning rate $0.1$. 95\% c.i.}
\label{consensusmore}
\end{figure}
\newpage
\subsection{Matching pennies linear parametrization}
In this section we reported the behavior of \algname~and the other benchmarks in a linear parametrization of Matching Pennies game. As you can see \algname, as CGD, converges to the Nash Equilibrium in only one step. We show the best results searching between learning rates $1.0, 0.5, 0.1, 0.05$.
\begin{figure}[h]
\begin{minipage}{0.5\textwidth}
\centering
\includegraphics[scale=.8]{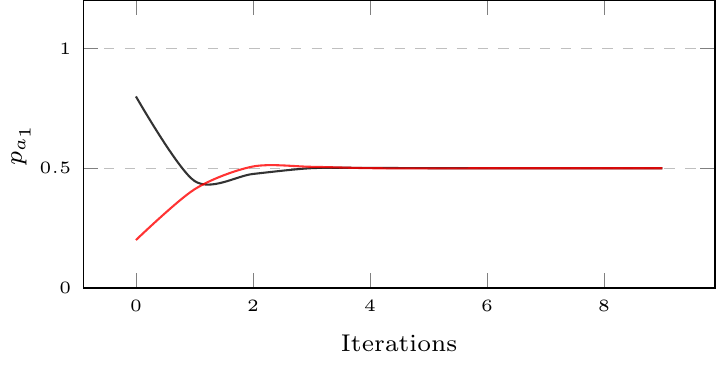}
\end{minipage}
\begin{minipage}{0.5\textwidth}
\centering
\includegraphics[scale=.8]{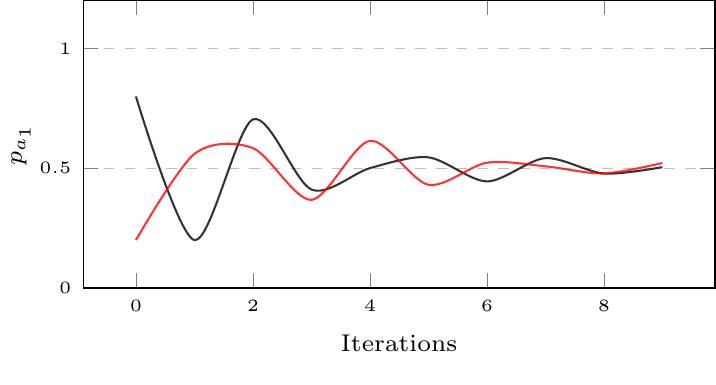}
\end{minipage}\\
\begin{minipage}{0.5\textwidth}
\centering
\includegraphics[scale=.8]{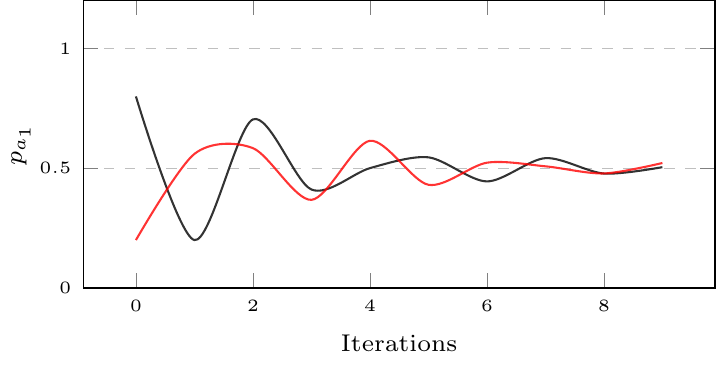}
\end{minipage}
\begin{minipage}{0.5\textwidth}
\centering
\includegraphics[scale=.8]{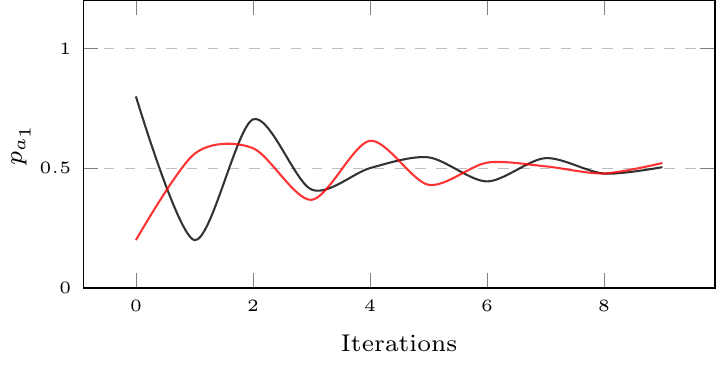}
\end{minipage}
\begin{minipage}{0.5\textwidth}
\centering
\includegraphics[scale=.8]{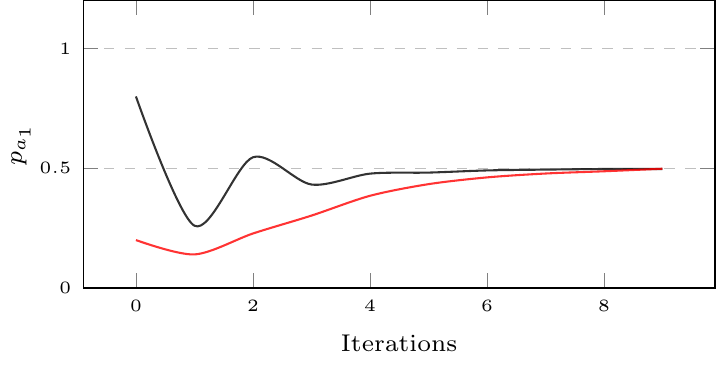}
\end{minipage}
\begin{minipage}{0.5\textwidth}
\centering
\includegraphics[scale=.8]{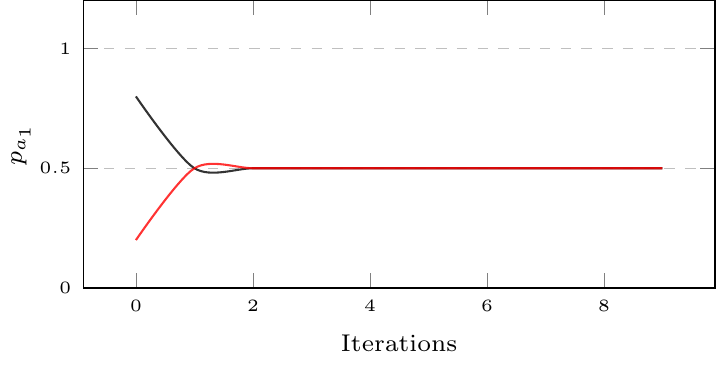}
\end{minipage}
\begin{minipage}{0.5\textwidth}
\centering
\includegraphics[scale=.8]{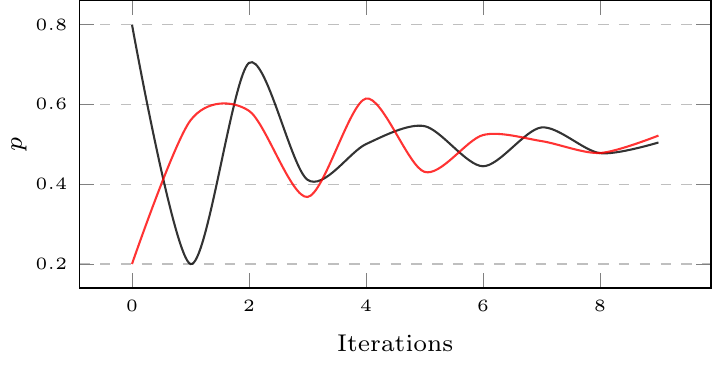}
\end{minipage}
\caption{Agents' probabily to perform action $1$ in Matching Pennies. The initial probabilities are settled to $0.8$ and $0.2$. From top right to bottom left: CGD, consensus, LOLA, IGAPP, SOS, \algname, and SGA.}
\label{matchingpennies1}
\end{figure}

\newpage

\subsection{Matching pennies with exact gradients}
In this section we reported the experiments on Matching Pennies game with $20$ different starting probabilities, sampled from a Normal distribution with mean $0$ and standard deviation $1$. Figure~\ref{matchingpennies1} shows how all the algorithms succeed in converging to the Nash Equilibrium, but NOHD converges in less than $100$ iterations.
\begin{figure}[h!]
\begin{minipage}{0.5\textwidth}
\centering
\includegraphics[scale=.8]{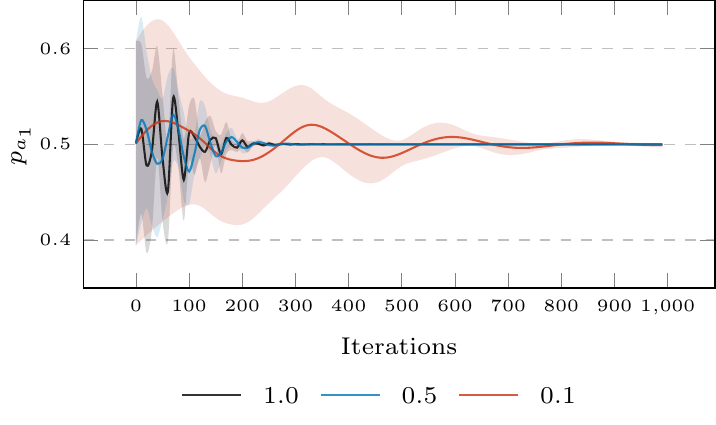}
\end{minipage}
\begin{minipage}{0.5\textwidth}
\centering
\includegraphics[scale=.8]{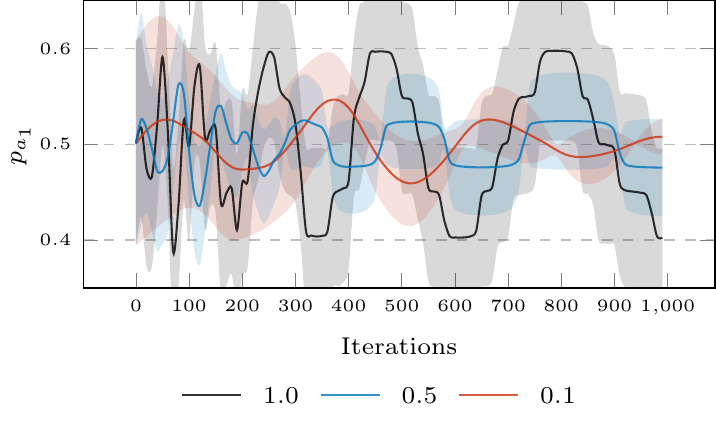}
\end{minipage}\\
\begin{minipage}{0.5\textwidth}
\centering
\includegraphics[scale=.8]{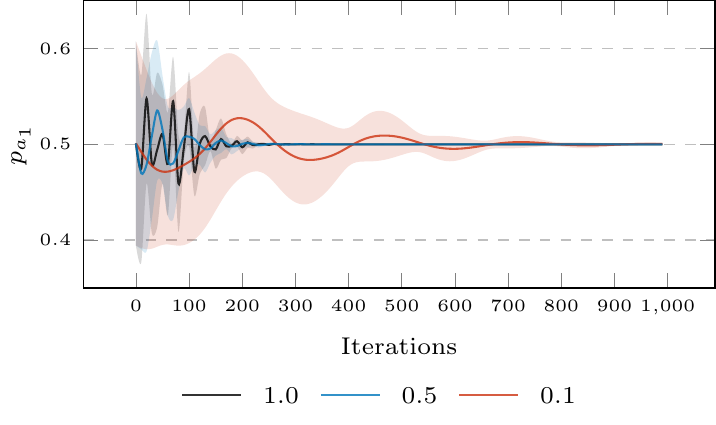}
\end{minipage}
\begin{minipage}{0.5\textwidth}
\centering
\includegraphics[scale=.8]{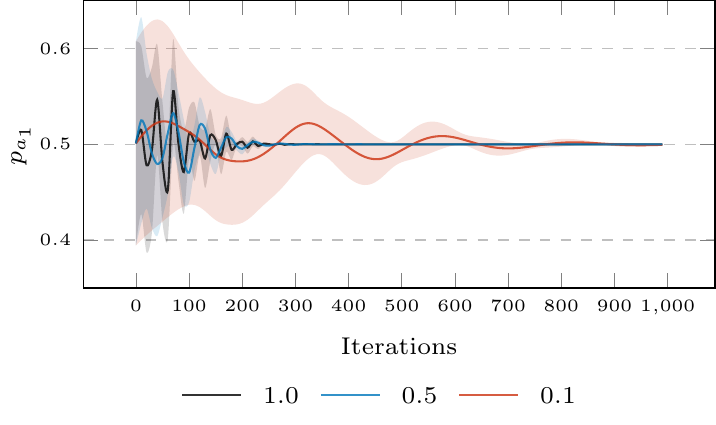}
\end{minipage}
\begin{minipage}{0.5\textwidth}
\centering
\includegraphics[scale=.8]{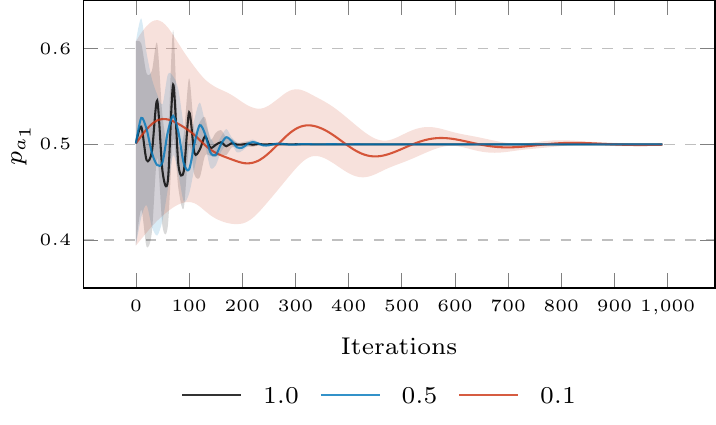}
\end{minipage}
\begin{minipage}{0.5\textwidth}
\centering
\includegraphics[scale=.8]{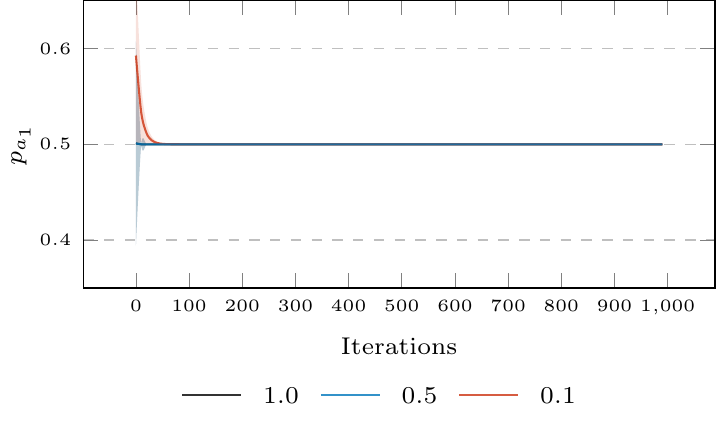}
\end{minipage}
\begin{minipage}{0.5\textwidth}
\centering
\includegraphics[scale=.8]{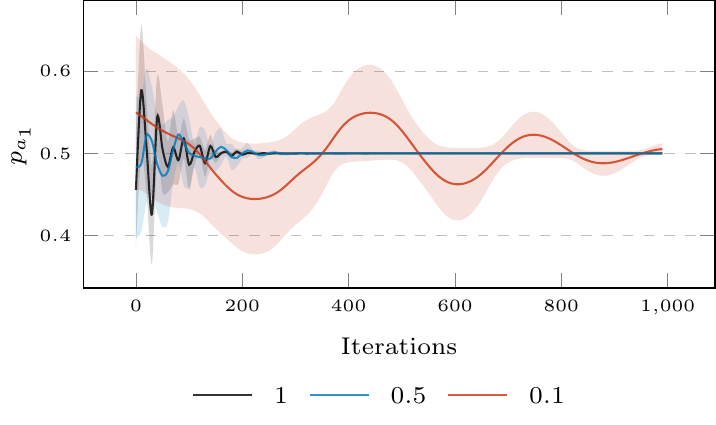}
\end{minipage}
\caption{Agent 1's probabily to perform action $1$ in Matching Pennies. From top right to bottom left: CGD, consensus, LOLA, IGAPP, SOS, \algname, and SGA. True gradients. $20$ random sampled initial probabilities. Learning rate $0.1, 0.5, 1.0$. 95\% c.i.}
\label{matchingpennies1}
\end{figure}

\newpage
\subsection{Matching pennies with approximated gradients}
In this section, we reported the experiments on the Matching Pennies game with $20$ different starting probabilities, sampled from a Normal distribution with mean $0$ and standard deviation $1$. In this case we estimate the gradient and the Jacobian using $300$ sampled trajectories. Figure~\ref{matchingpenniesapx1} shows how all the algorithms succeed in converging to the Nash Equilibrium, but \algname~converges in less than $100$.

\begin{figure}[h!]
\begin{minipage}{0.5\textwidth}
\centering
\includegraphics[scale=.8]{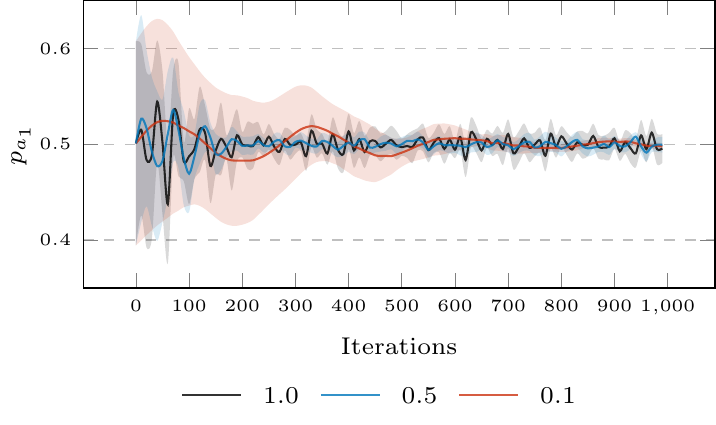}
\end{minipage}
\begin{minipage}{0.5\textwidth}
\centering
\includegraphics[scale=.8]{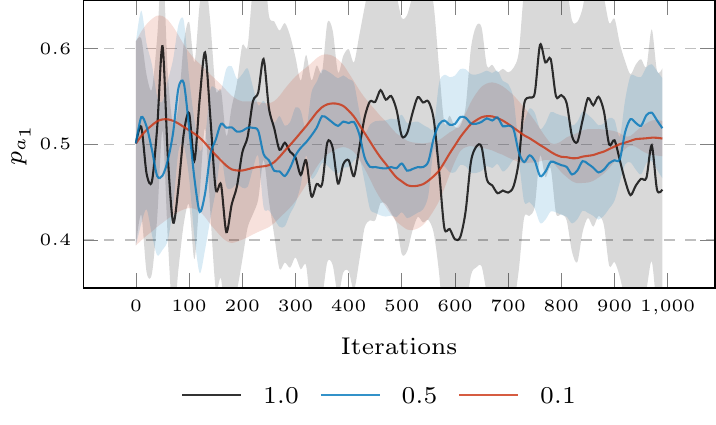}
\end{minipage}\\
\begin{minipage}{0.5\textwidth}
\centering
\includegraphics[scale=.8]{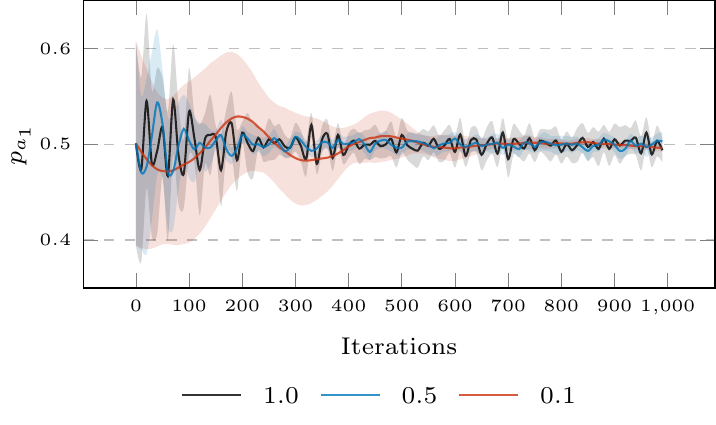}
\end{minipage}
\begin{minipage}{0.5\textwidth}
\centering
\includegraphics[scale=.8]{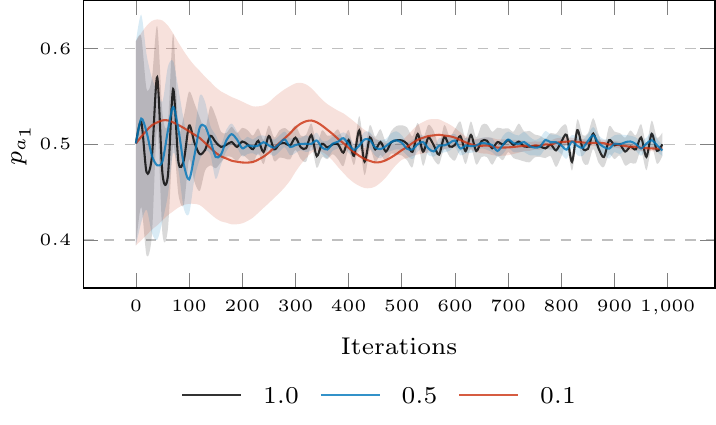}
\end{minipage}
\begin{minipage}{0.5\textwidth}
\centering
\includegraphics[scale=.8]{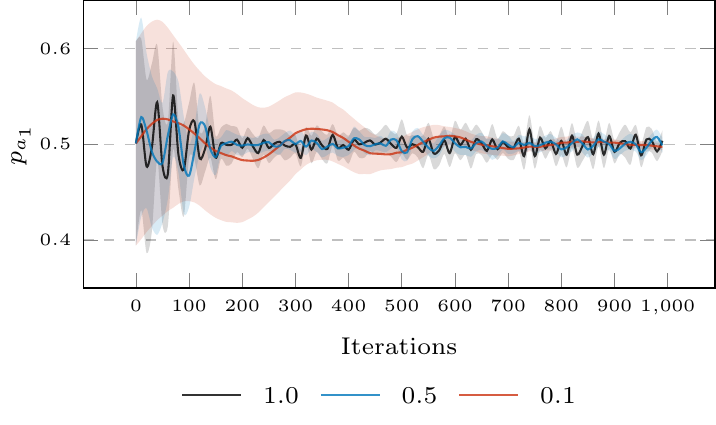}
\end{minipage}
\begin{minipage}{0.5\textwidth}
\centering
\includegraphics[scale=.8]{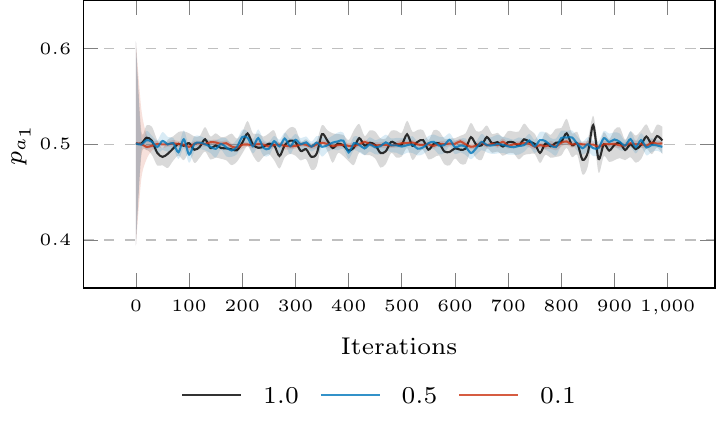}
\end{minipage}
\begin{minipage}{0.5\textwidth}
\centering
\includegraphics[scale=.8]{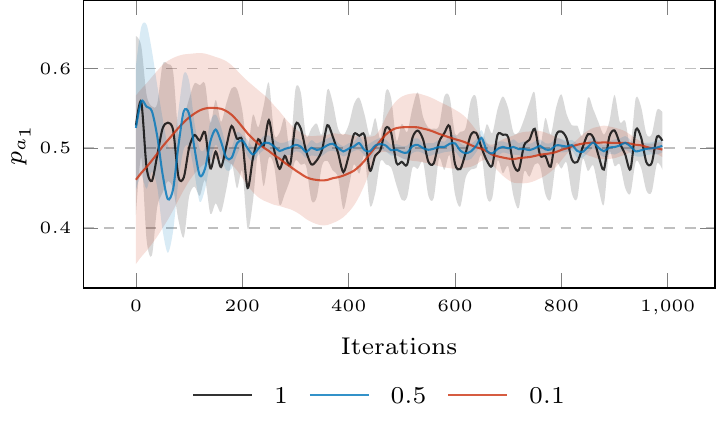}
\end{minipage}
\caption{Agent 1's probabily to perform action $1$ in Matching Pennies. From top right to bottom left: CGD, consensus, LOLA, IGAPP, SOS, \algname, and SGA. Estimated gradients batch size $300$. $20$ random sampled initial probabilities. Learning rate $0.1, 0.5, 1.0$. 95\% c.i.}
\label{matchingpenniesapx1}
\end{figure}

\newpage

\subsection{Rock paper scissors with exact gradients}
In this section, we reported the experiments on Matching Pennies game with $20$ different starting probabilities, sampled from a Normal distribution with mean $0$ and standard deviation $0.5$. For every algorithm we perform experiments with learning rates $0.1, 0.5, 1$. Figure~\ref{rps1} shows how all the algorithms succeed in converging to the Nash Equilibrium, but \algname~converges in less than $100$ iterations.
\begin{figure}[h!]
\begin{minipage}{0.5\textwidth}
\centering
\includegraphics[scale=.8]{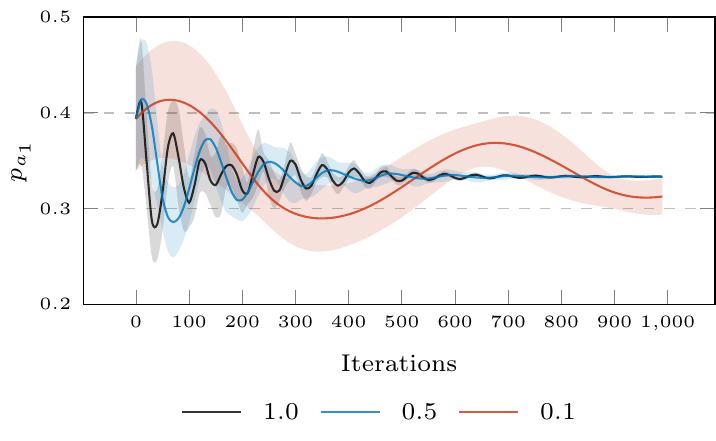}
\end{minipage}
\begin{minipage}{0.5\textwidth}
\centering
\includegraphics[scale=.8]{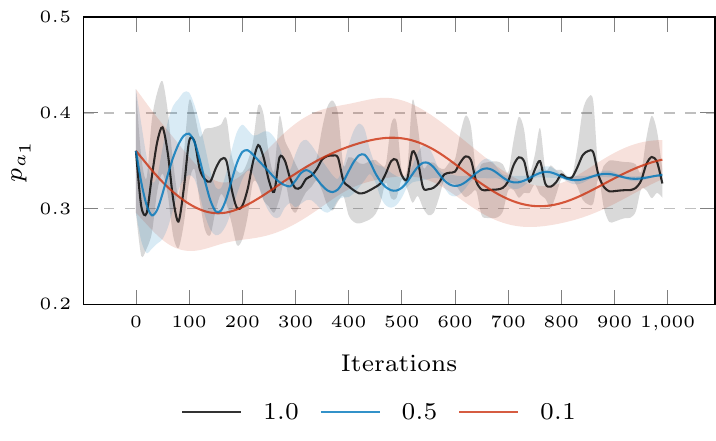}
\end{minipage}\\
\begin{minipage}{0.5\textwidth}
\centering
\includegraphics[scale=.8]{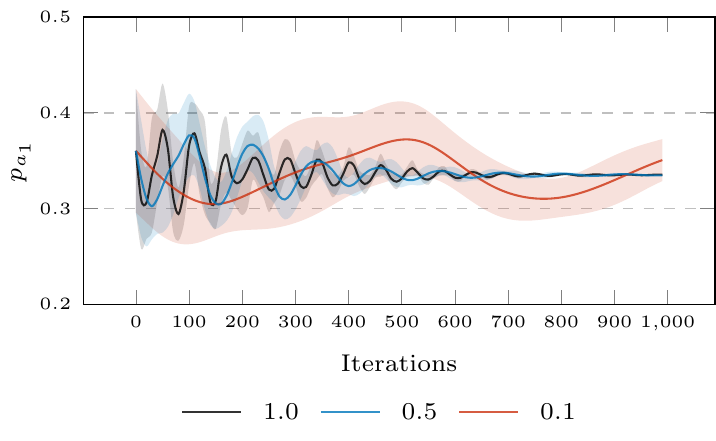}
\end{minipage}
\begin{minipage}{0.5\textwidth}
\centering
\includegraphics[scale=.8]{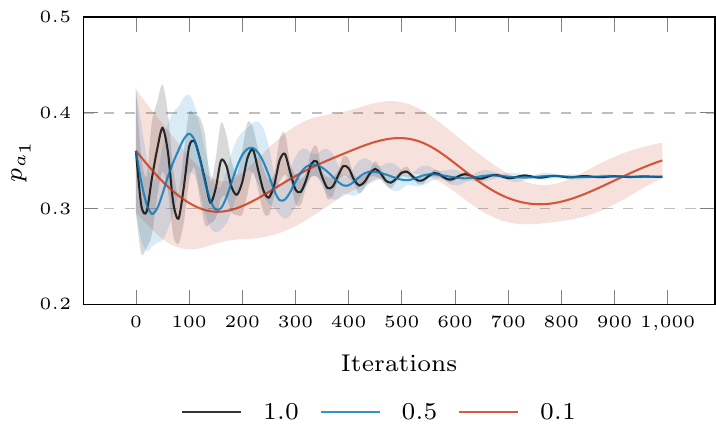}
\end{minipage}
\begin{minipage}{0.5\textwidth}
\centering
\includegraphics[scale=.8]{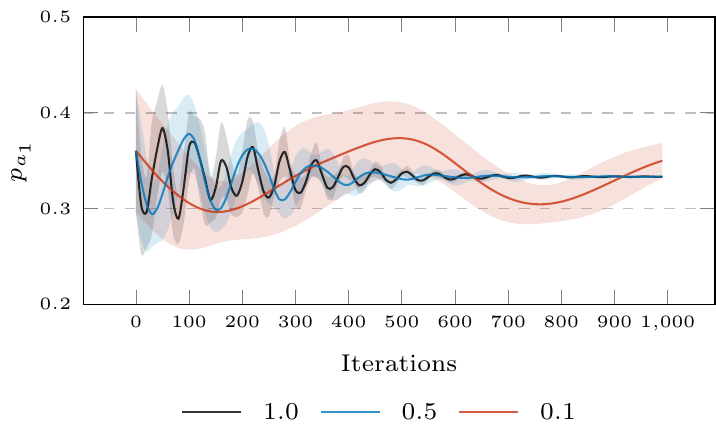}
\end{minipage}
\begin{minipage}{0.5\textwidth}
\centering
\includegraphics[scale=.8]{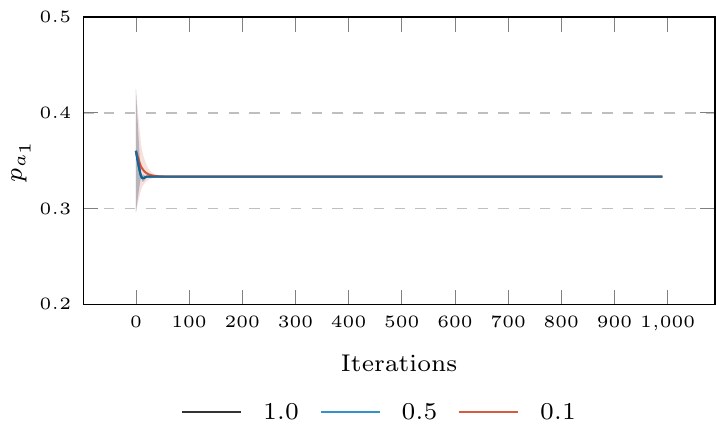}
\end{minipage}
\begin{minipage}{0.5\textwidth}
\centering
\includegraphics[scale=.8]{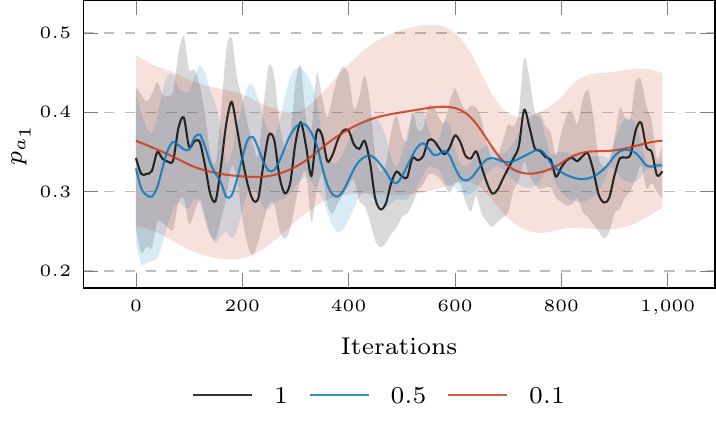}
\end{minipage}
\caption{Agent 1's probabily to perform action $1$ in Rock Paper Scissors. From top right to bottom left: CGD, consensus, LOLA, IGAPP, SOS, \algname, and SGA. Estimated gradients batch size $300$. $20$ random sampled initial probabilities. Learning rate $0.1, 0.5, 1.0$. 95\% c.i.}
\label{rps1}
\end{figure}

\newpage
\subsection{Rock Paper Scissors with approximated gradients}
In this section, we reported the experiments on Rock--Paper--Scissors game with $20$ different starting probabilities, sampled from a Normal distribution with mean $0$ and standard deviation $0.5$. In this case, we estimate the gradient and the Jacobian using $300$ sampled trajectories. Every algorithms do not perform well with learning rate $1.0$ and for this reason we report results only with learning rates $0.1, 0.5$. Figure~\ref{rpsapx1} shows how all the algorithms succeed in converging to the Nash Equilibrium, but \algname converges in less than $100$.
\begin{figure}[h!]
\begin{minipage}{0.5\textwidth}
\centering
\includegraphics[scale=.8]{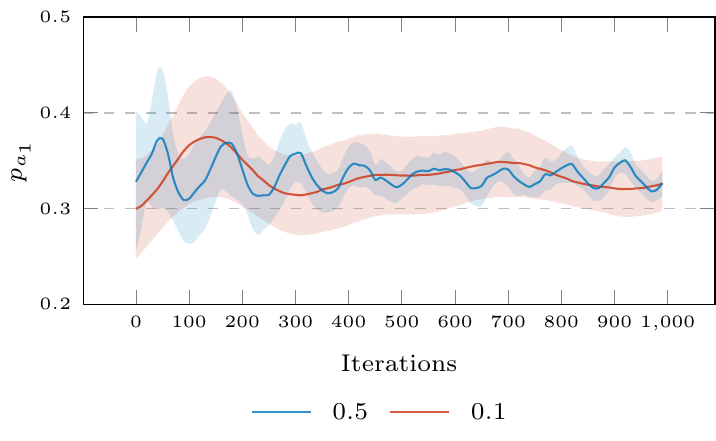}
\end{minipage}
\begin{minipage}{0.5\textwidth}
\centering
\includegraphics[scale=.8]{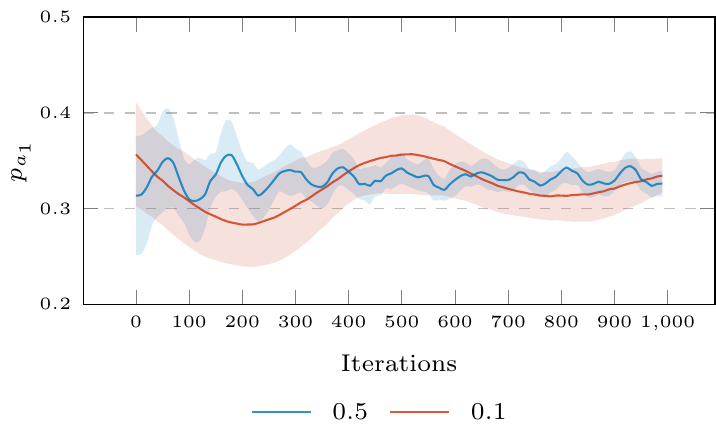}
\end{minipage}\\
\begin{minipage}{0.5\textwidth}
\centering
\includegraphics[scale=.8]{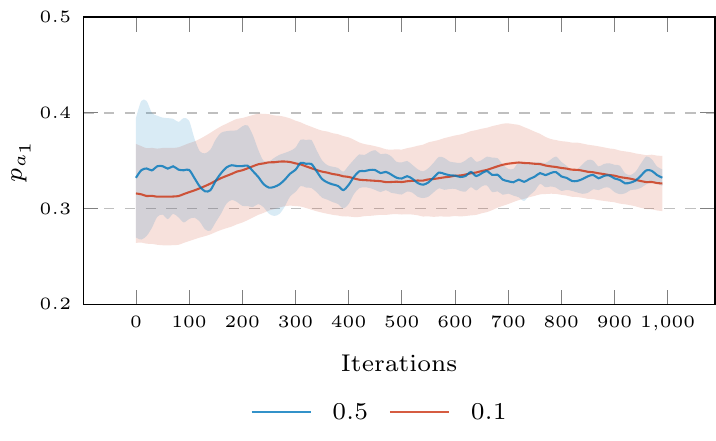}
\end{minipage}
\begin{minipage}{0.5\textwidth}
\centering
\includegraphics[scale=.8]{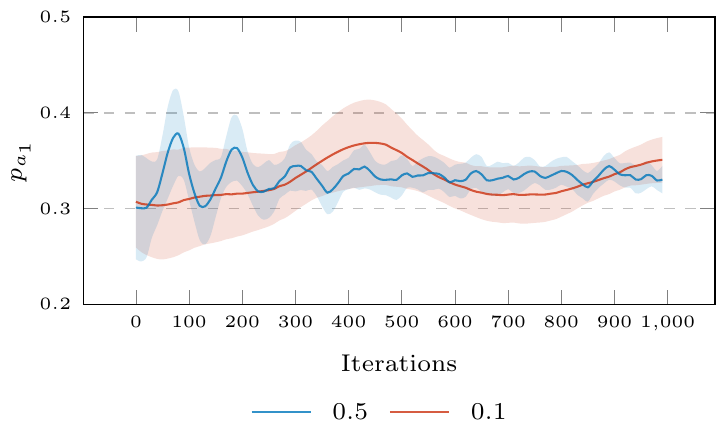}
\end{minipage}
\begin{minipage}{0.5\textwidth}
\centering
\includegraphics[scale=.8]{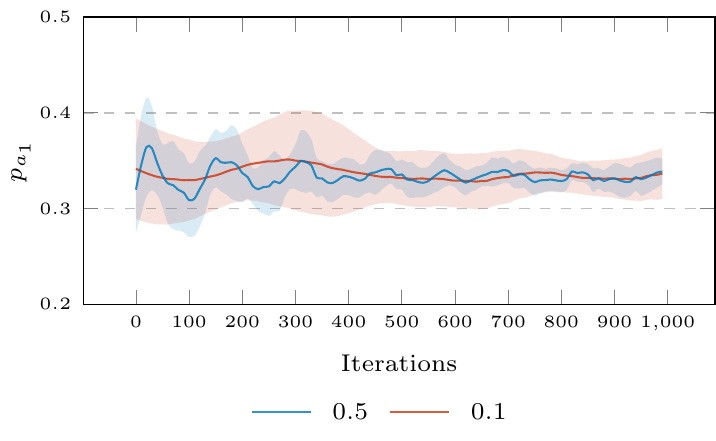}
\end{minipage}
\begin{minipage}{0.5\textwidth}
\centering
\includegraphics[scale=.8]{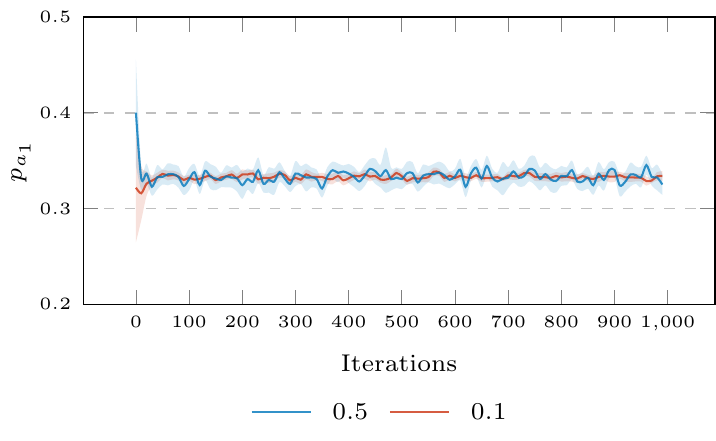}
\end{minipage}
\begin{minipage}{0.5\textwidth}
\centering
\includegraphics[scale=.8]{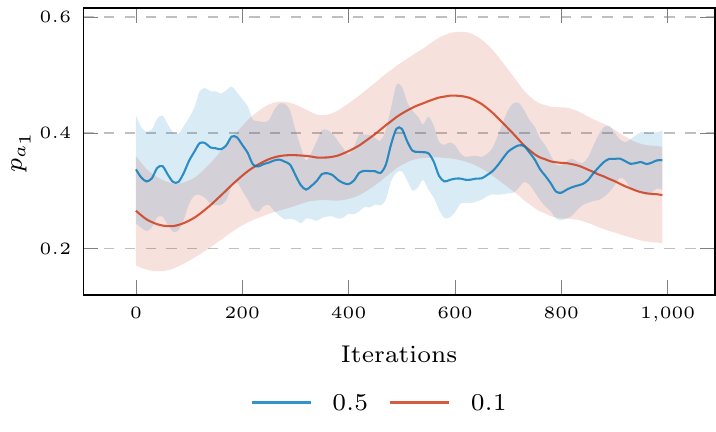}
\end{minipage}
\caption{Agent 1's probabily to perform action $1$ in Rock Paper Scissors. From top right to bottom left: CGD, consensus, LOLA, IGAPP, SOS, \algname, and SGA. Estimated gradients. $20$ random sampled initial probabilities. Learning rate $0.1, 0.5$. 95\% c.i.}
\label{rpsapx1}
\end{figure}

\newpage

\subsection{Continuous gridworlds}
In this section we report some details on the setting of gridworld experiments. In this experiment the agent is described by a gaussian policy with variance $0.1$ and mean features are respectively $72$ and $68$ radial basis function which describes the two agent position, and in the second gridworld, which agent has the ball. In the first gridworld the reward that an agent take is $0$ for every state except the goal state where is $50$. In the second is $-1$ for every state and $3$ in the goal state. The hyperparameters used for these experiments are:

\begin{itemize}
\item Learning rate: $0.01$
\item Batch size: $100$
\item Trajectory length: $30$
\item Discount factor: $0.96$
\end{itemize}

\begin{figure}
\includegraphics{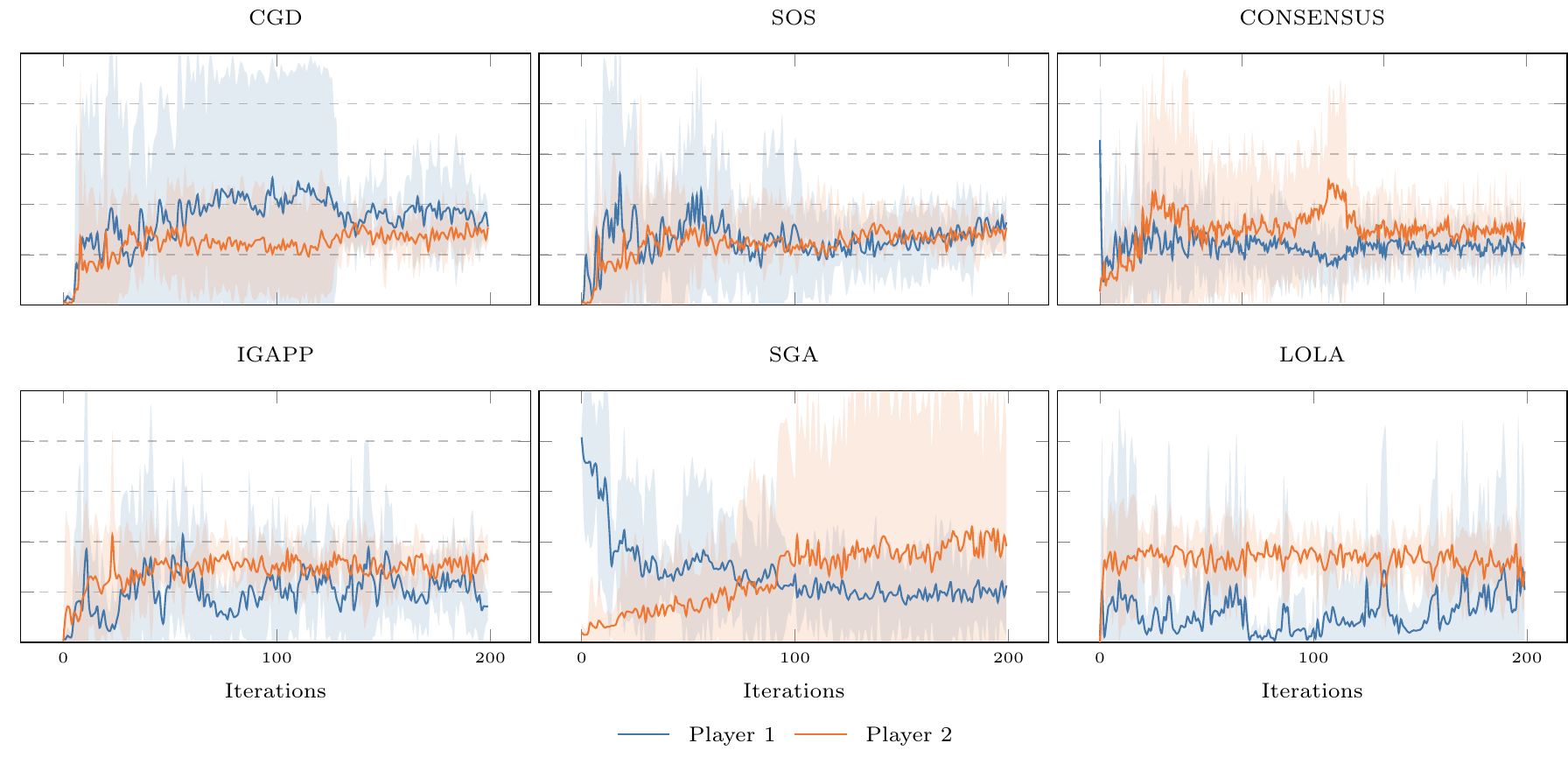}
\label{fig:nohd}
\caption{Average wins for player $1$ and player $2$ in the second gridworld environments, where the second player is always NOHD. $10$ runs, c.i.~98\%}
\end{figure}

Below we report the second experiment where the second player is always NOHD. Figure \ref{fig:nohd} shows that all algorithms converge to a stable policy against NOHD or NOHD learn a winner policy against the algorithm. 
\newpage
\subsection{Generative Adversarial Network}
\begin{figure*}[h!]
\begin{minipage}{0.07\textwidth}
\textbf{NOHD}
\end{minipage}
\begin{minipage}{0.17\textwidth}
\includegraphics[scale=.21]{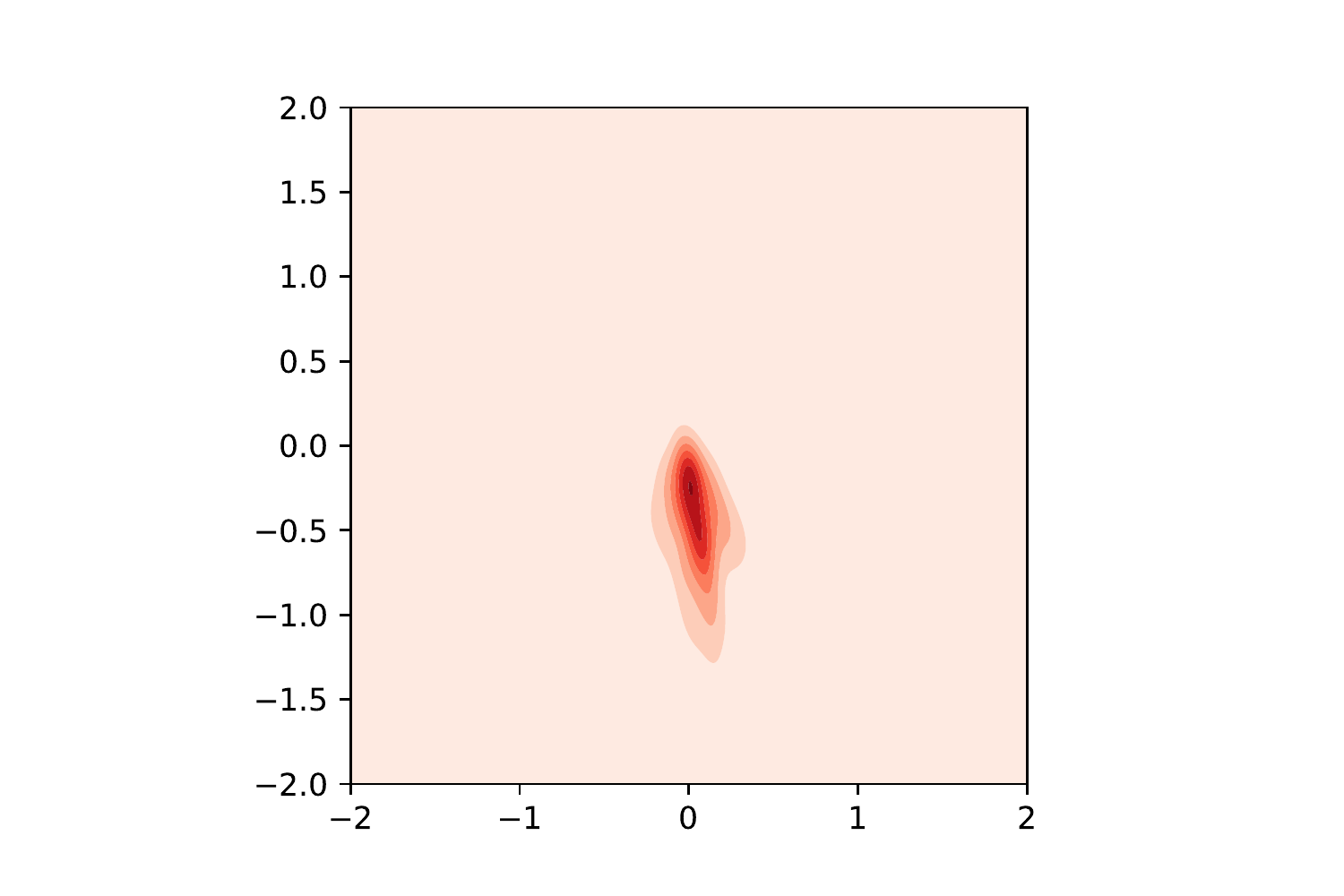}
\end{minipage}
\begin{minipage}{0.17\textwidth}
\includegraphics[scale=.21]{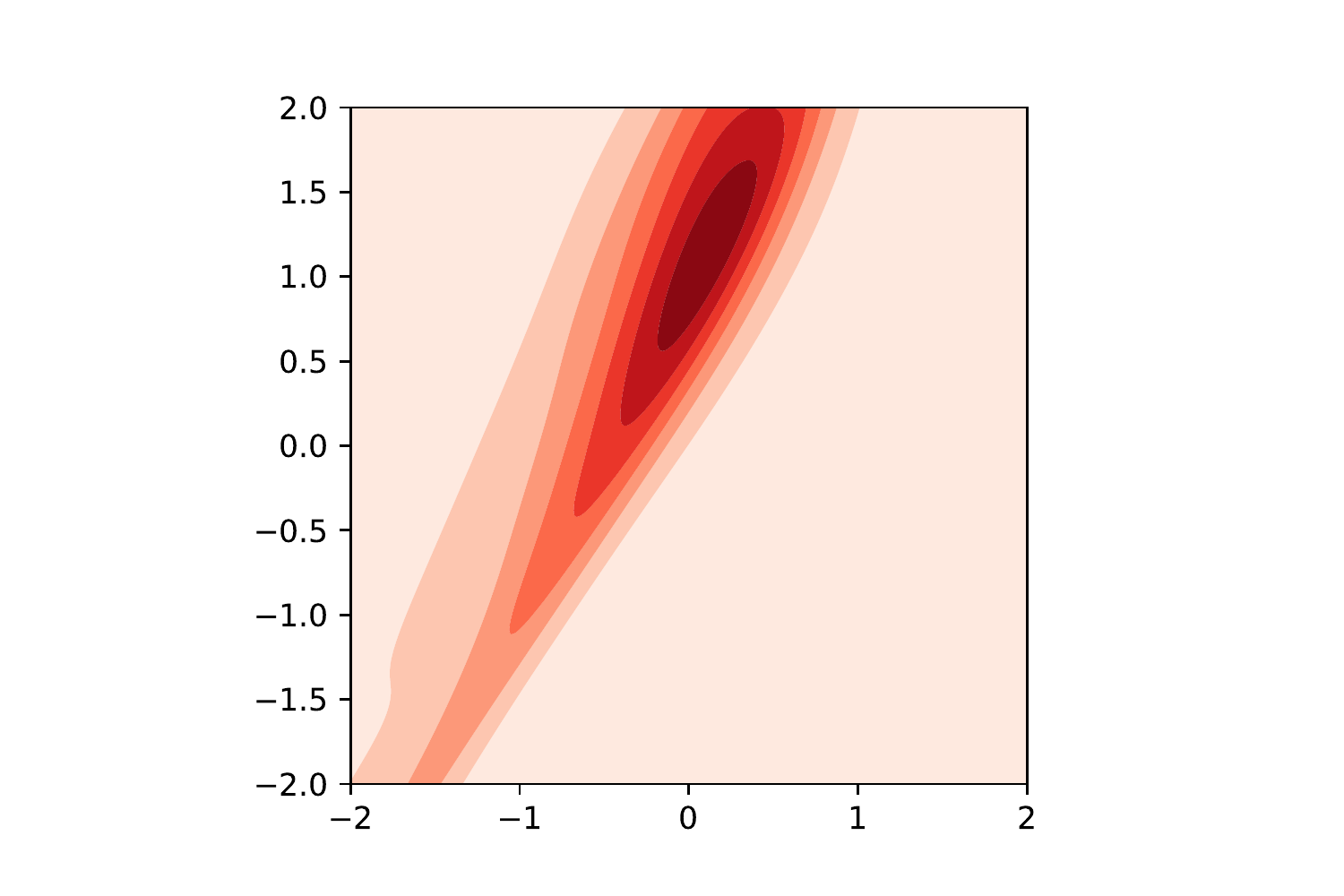}
\end{minipage}
\begin{minipage}{0.17\textwidth}
\includegraphics[scale=.21]{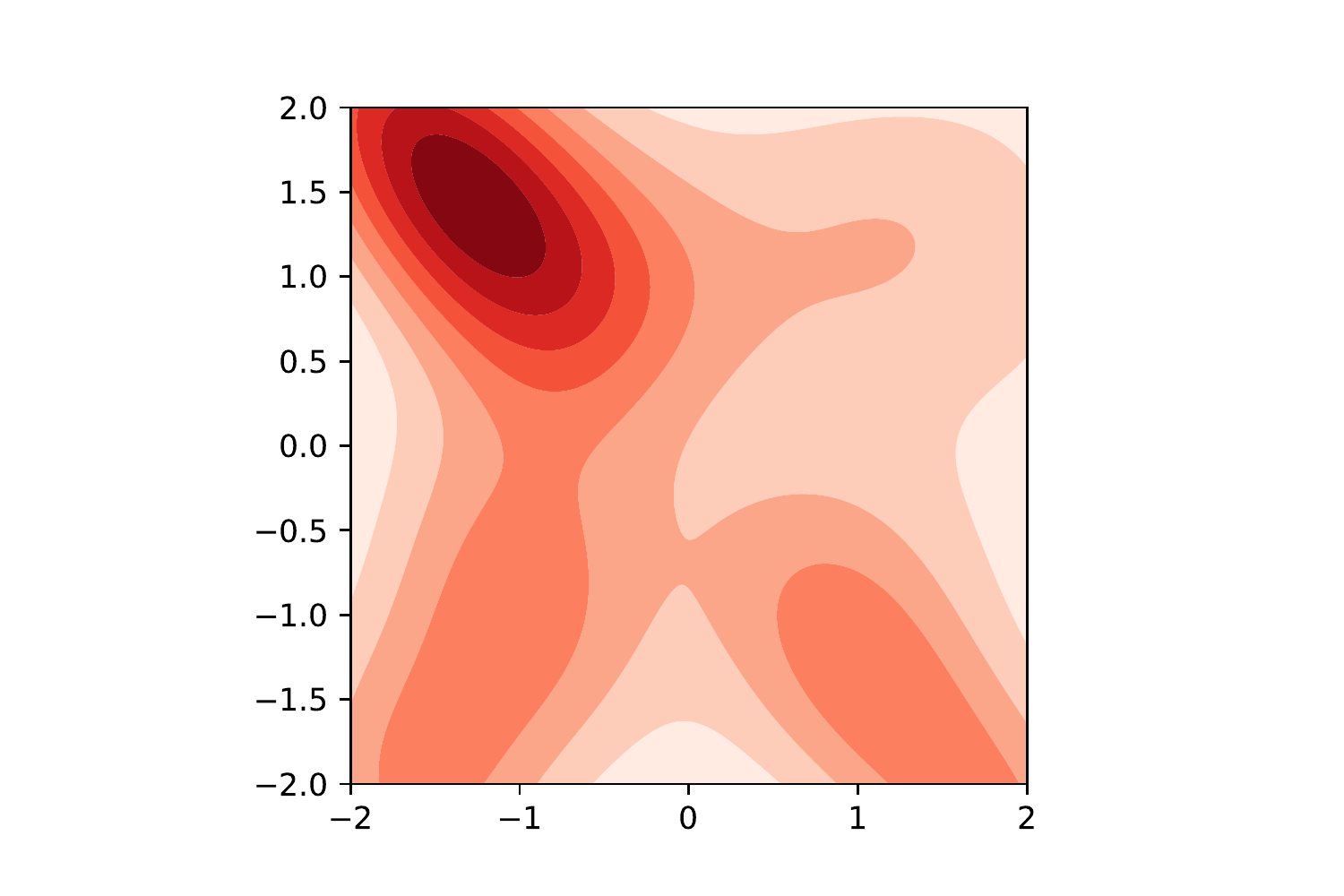}
\end{minipage}
\begin{minipage}{0.17\textwidth}
\includegraphics[scale=.21]{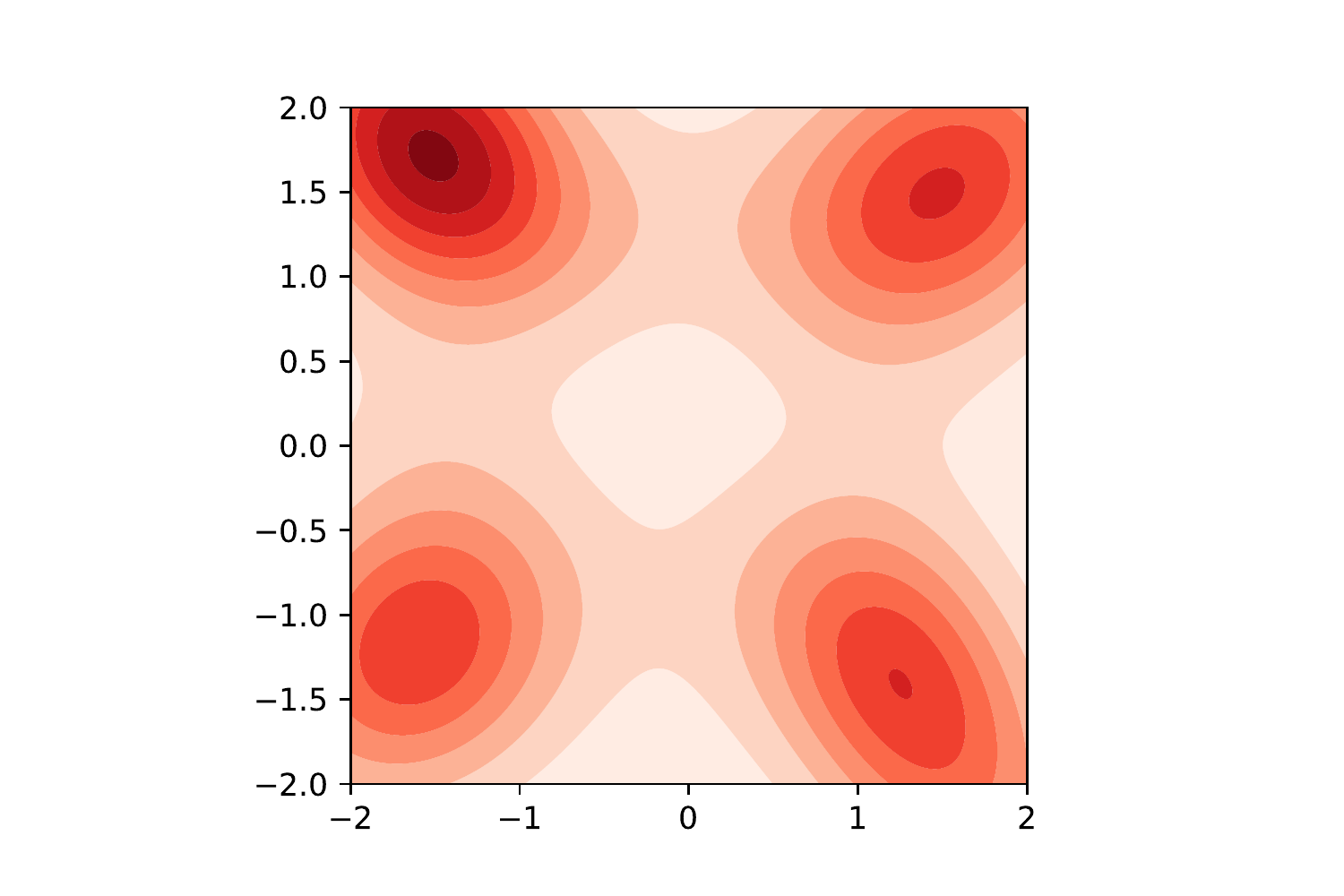}
\end{minipage}
\begin{minipage}{0.17\textwidth}
\includegraphics[scale=.21]{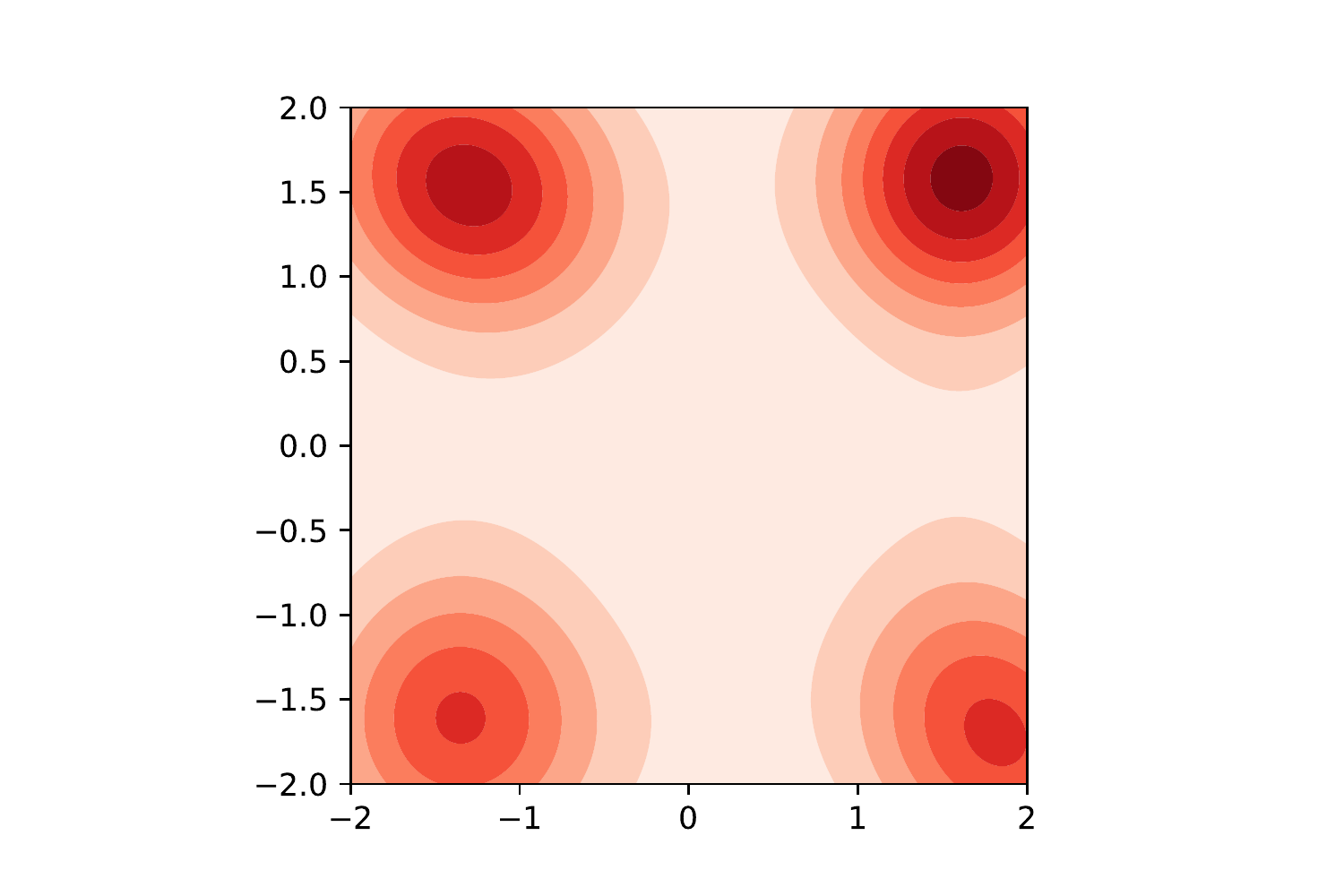}
\end{minipage}
\\
\begin{minipage}{0.07\textwidth}
\textbf{SGA}
\end{minipage}
\begin{minipage}{0.17\textwidth}
\includegraphics[scale=.21]{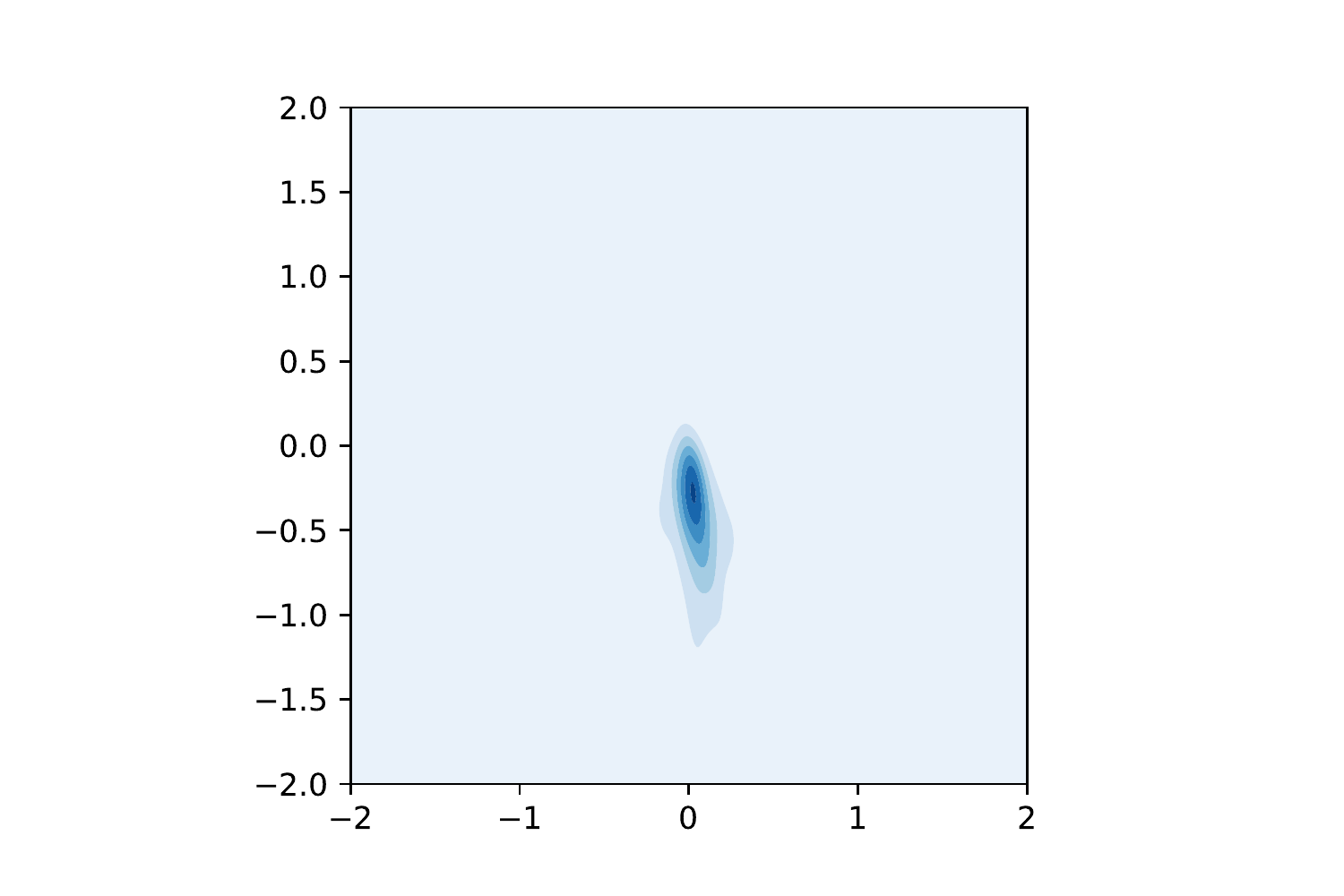}
\end{minipage}
\begin{minipage}{0.17\textwidth}
\includegraphics[scale=.21]{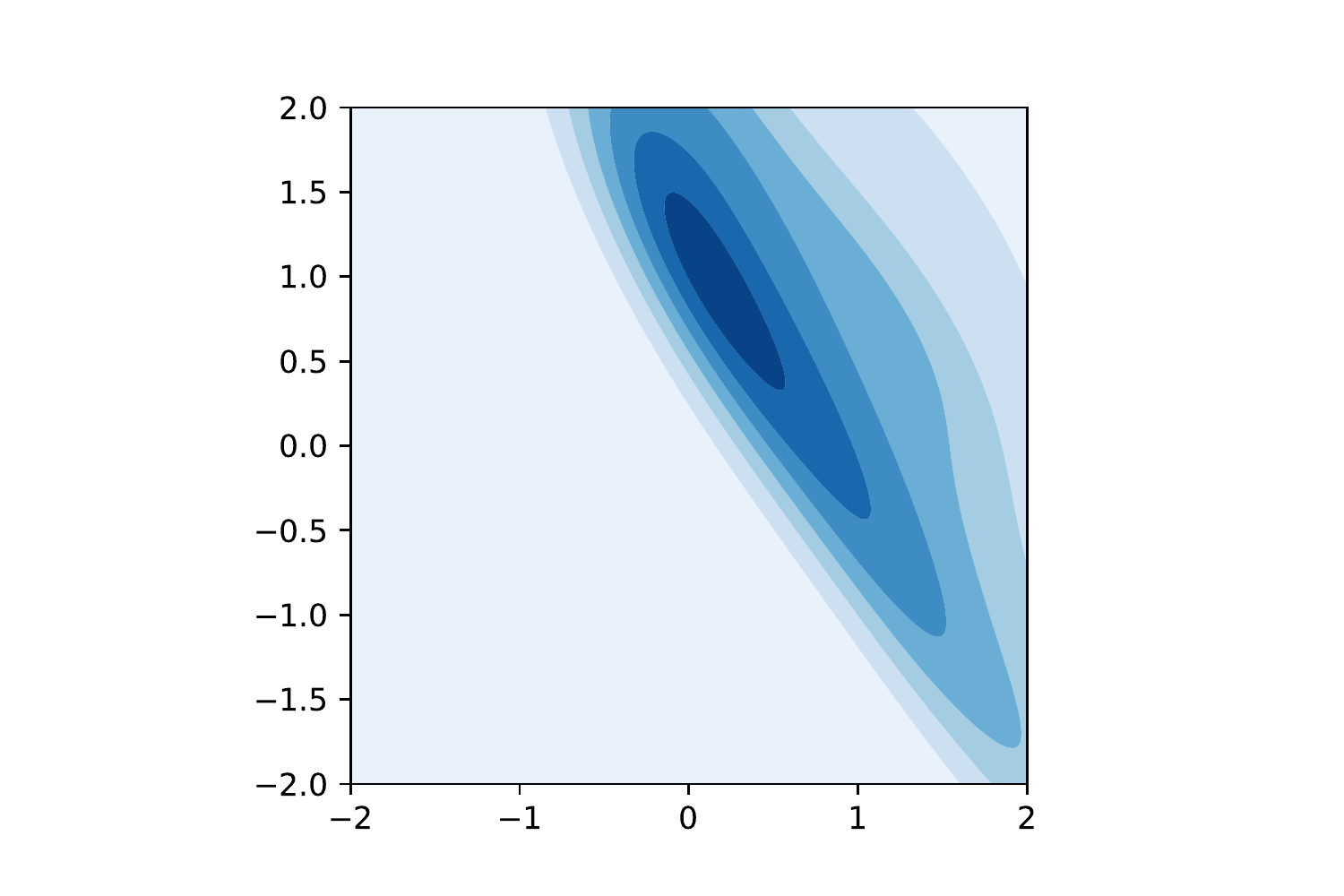}
\end{minipage}
\begin{minipage}{0.17\textwidth}
\includegraphics[scale=.21]{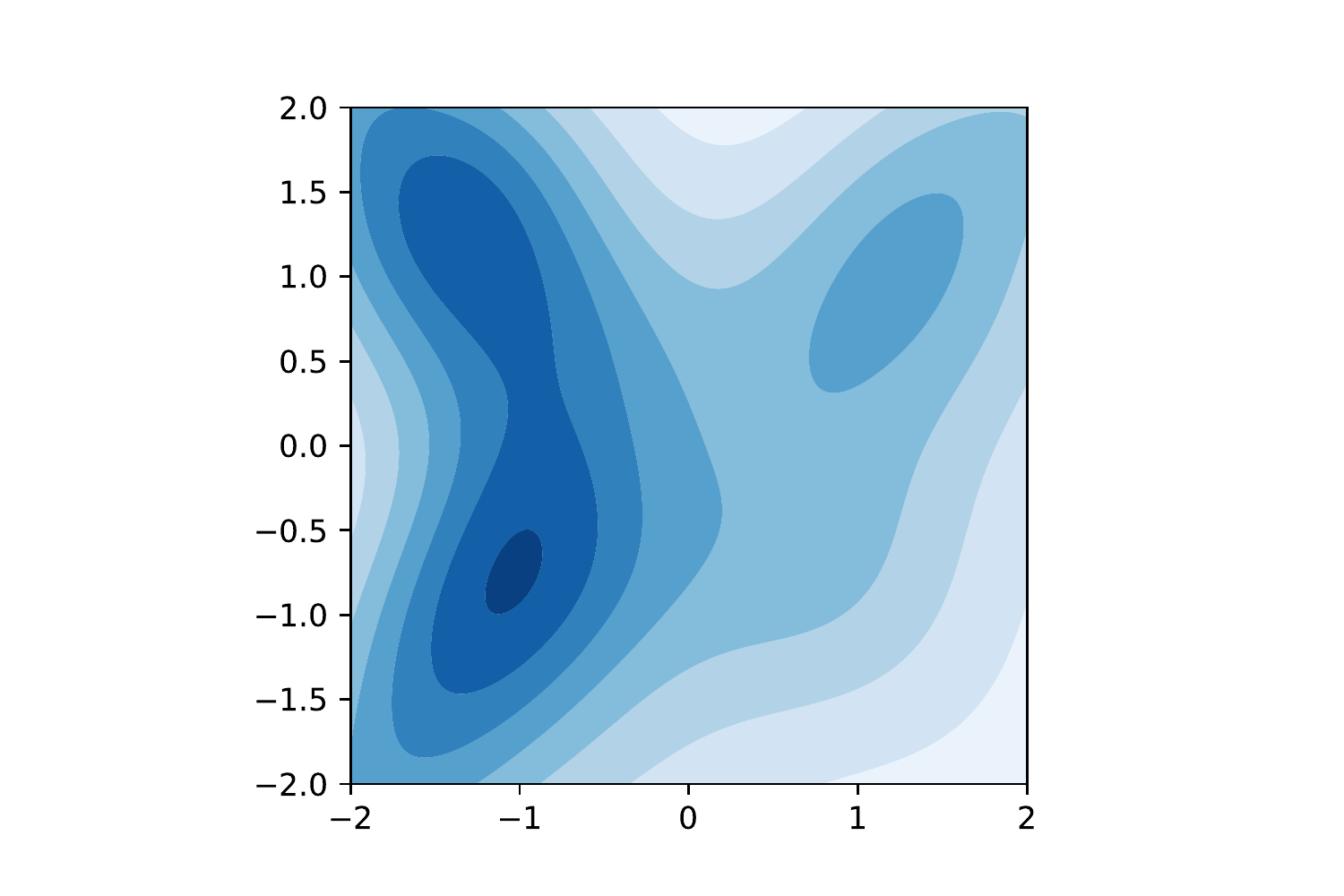}
\end{minipage}
\begin{minipage}{0.17\textwidth}
\includegraphics[scale=.21]{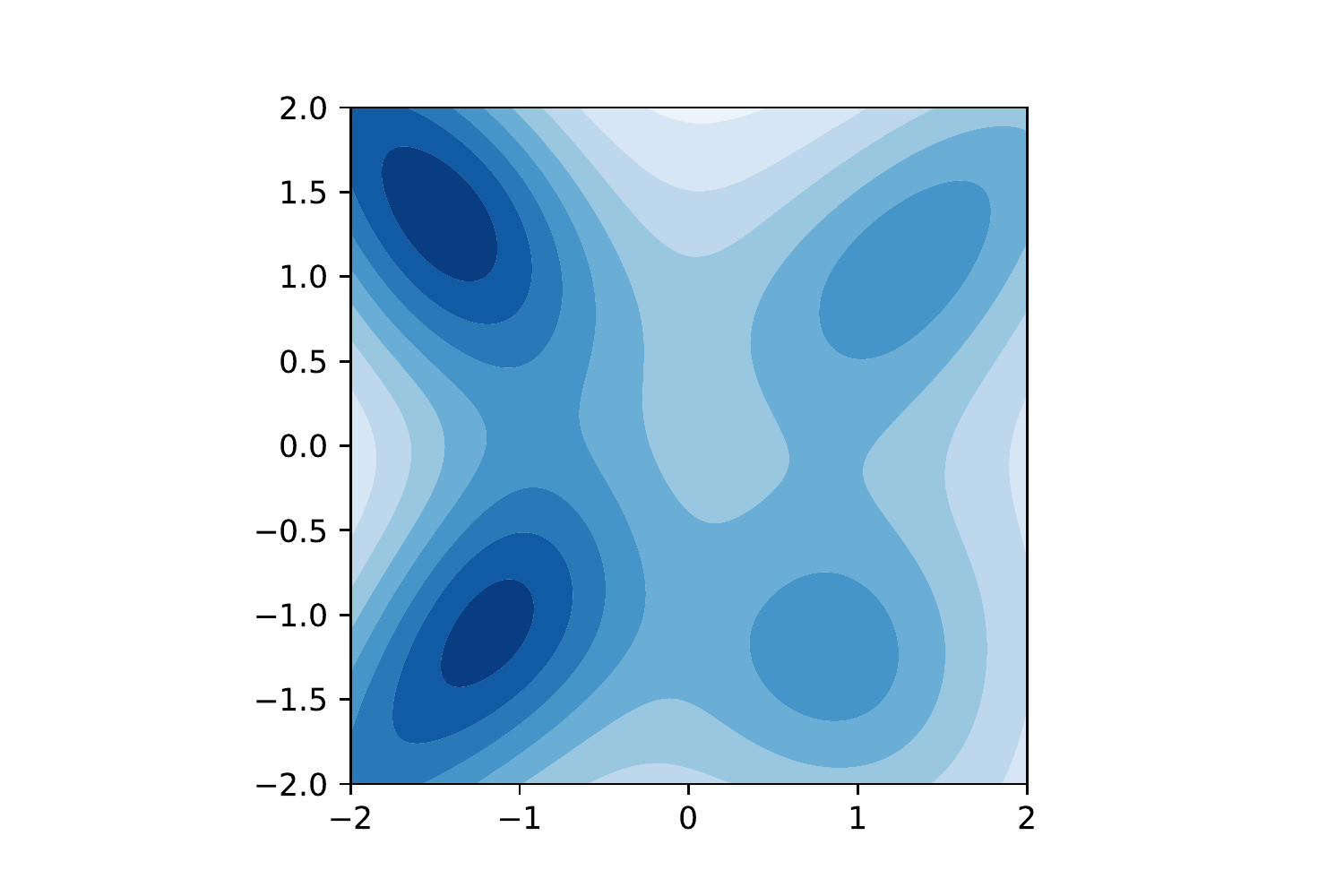}
\end{minipage}
\begin{minipage}{0.17\textwidth}
\includegraphics[scale=.21]{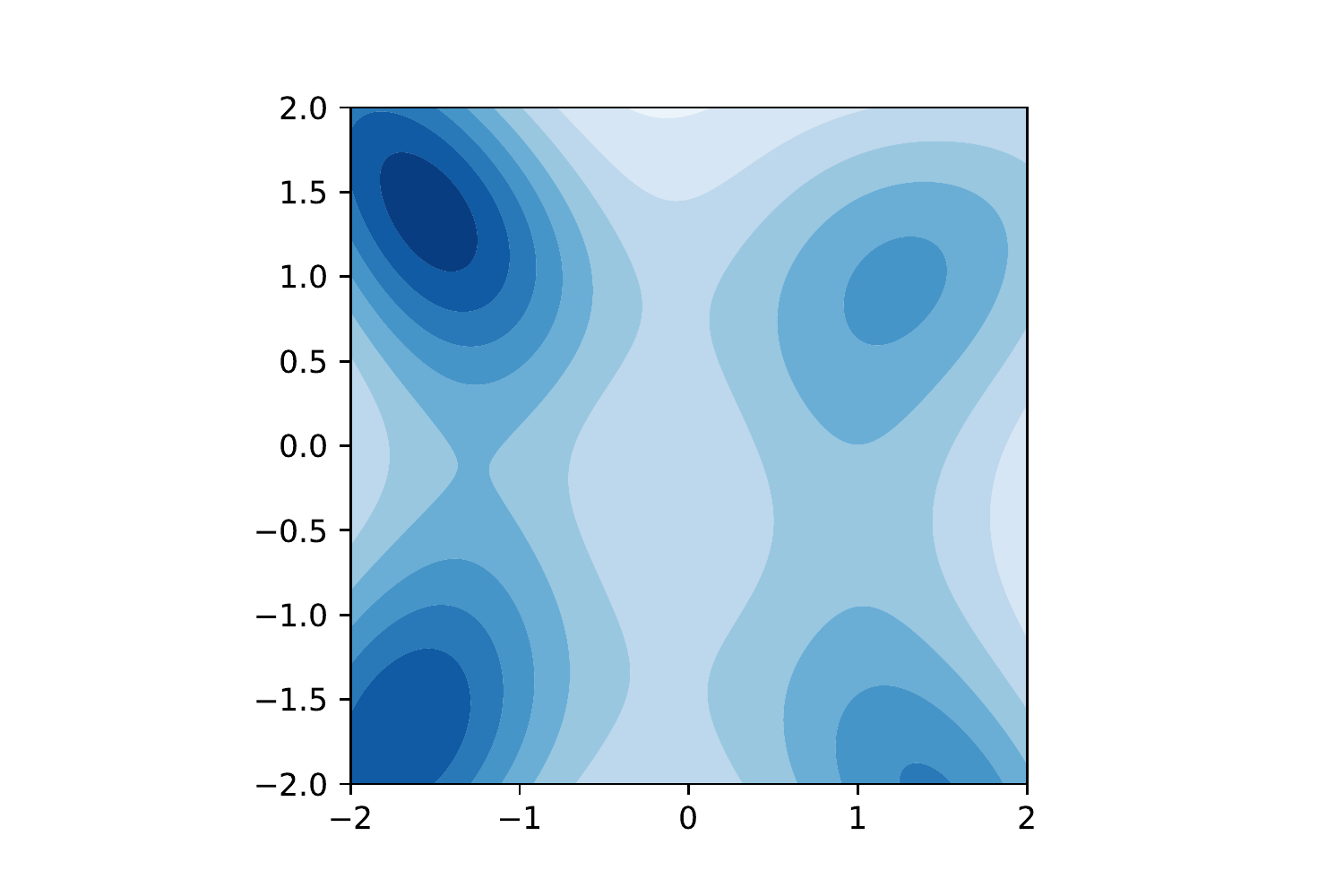}
\end{minipage}
\caption{Generator's learnt distribution for iterations $0$, $100$, $200$, $300$, $400$.}
\end{figure*}

We show a simple experiment with a mixture of $4$ bi-variate Gaussians with means: $(1.5,-1,5)$, $(1.5,1,5)$, $(-1.5,1.5)$, $(-1.5,-1.5)$. The generator and discriminator networks are both with two ReLu layers with $10$ neurons per layer. The output of the discriminator has size $1$ and the output of the generator has size $2$. The learning rate is $0.01$. We report that NOHD finds all the modes in $400$ steps and we compare these results with SGA. The results shown below are for random seed $25$.

\end{document}